\documentclass{article}

\usepackage{graphicx}
\usepackage[preprint]{neurips_2026} 

\usepackage[utf8]{inputenc}
\usepackage[T1]{fontenc}
\usepackage{hyperref}
\usepackage{url}
\usepackage{booktabs}
\usepackage{multirow}
\usepackage{amsfonts}
\usepackage{amsmath}
\usepackage{amsthm}
\usepackage{caption}
\usepackage{subcaption}
\usepackage{wrapfig}
\usepackage{nicefrac}
\usepackage{microtype}
\usepackage[table]{xcolor}
\usepackage[most]{tcolorbox}

\usepackage{enumitem}

\usepackage{etoc}
\usepackage{titletoc}
\title{LLMs Should Express Uncertainty Explicitly}

\newtheorem{theorem}{Theorem}
\newtheorem{proposition}{Proposition}
\newtheorem{corollary}{Corollary}
\author{%
  Junyu Guo\thanks{\noindent\textbf{Corresponding authors:} Junyu Guo (\texttt{junyuguo24@berkeley.edu}) and Ming Jin (\texttt{jinming@vt.edu}).} \\
  University of California, Berkeley
  \And
  Shangding Gu \\
  University of California, Berkeley
  \And
  Ming Jin$^*$ \\
  Virginia Tech
  \AND
  Costas Spanos \\
  University of California, Berkeley
  \And
  Javad Lavaei \\
  University of California, Berkeley
}

\newcommand{\utok}{%
  \texttt{\textcolor{blue!65!black}{\textbf{<uncertain>}}}%
}

\begin{document}

\maketitle

\begin{abstract}
Large language models (LLMs) often produce confident yet incorrect answers, which can lead to risky failures in real-world applications. We study whether post-training can make a model's self-assessment explicit: when the model is uncertain, can it be trained to signal so within its own response? A central design question is \emph{where} in the response this signal should be exposed --- during reasoning, while the answer is still being formed, or at the end, once the answer has been produced. We study both. For end-of-reasoning self-assessment, we train the model to verbalize a confidence score for its response, with the aim of high confidence on correct answers and low confidence on incorrect ones. For during-reasoning self-assessment, we train the model to emit the marker \texttt{<uncertain>} whenever its current reasoning state appears unreliable. Across factual reasoning tasks, both forms sharply reduce overconfident errors while improving answer quality, and both can be used as triggers for retrieval augmented generation (RAG)  to improve the final response. We further analyze their internal mechanisms: end-of-reasoning verbalized confidence sharpens a confidence-related structure already present in the pretrained model, whereas during-reasoning \texttt{<uncertain>} emission teaches the model to mark high-risk reasoning steps, with parameter changes concentrated in the model's late layers.

\end{abstract}

\begin{figure*}[htbp]
    \centering
    \includegraphics[width=0.85\textwidth]{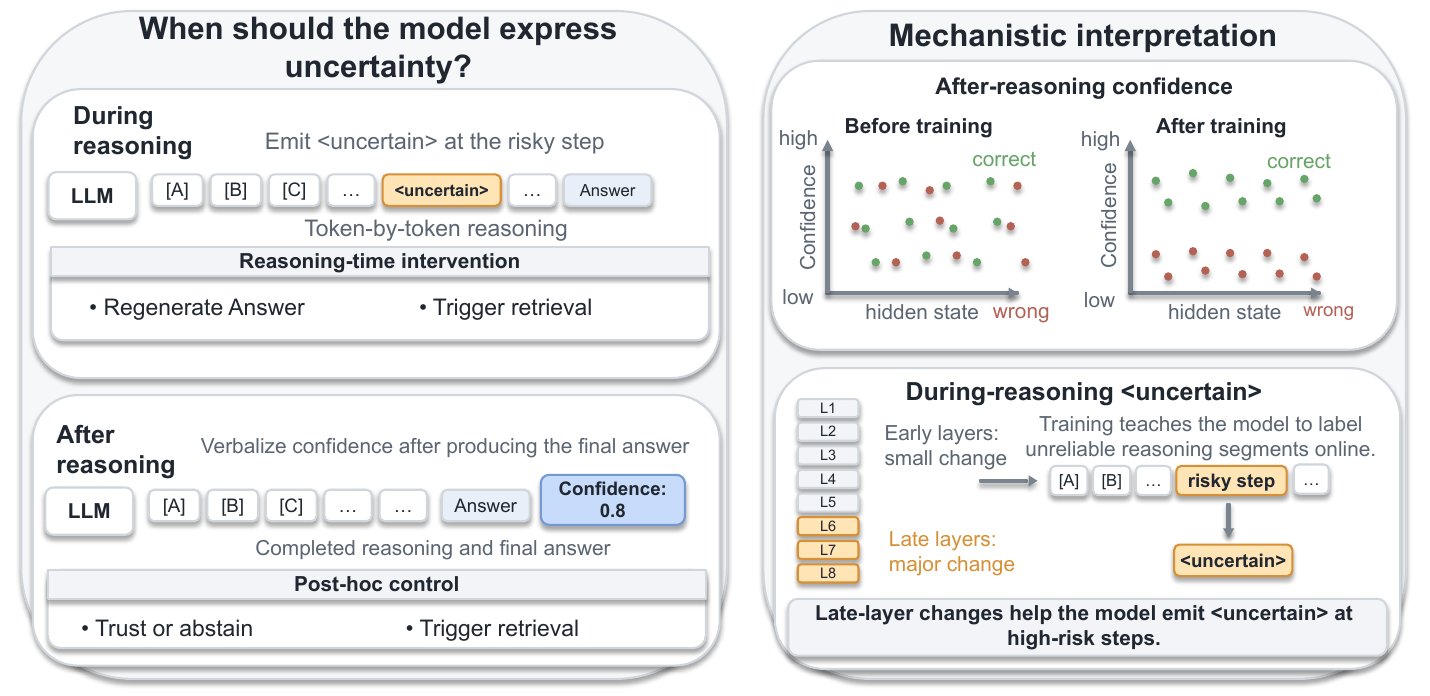}
 \caption{
We train LLMs to express uncertainty at two points in the response: during reasoning by emitting \texttt{<uncertain>} at risky steps, and after reasoning by verbalizing a confidence score for the final answer. 
Both signals can trigger retrieval or abstention, while our analysis suggests that confidence training sharpens an existing confidence-related structure and \texttt{<uncertain>} training teaches the model to mark high-risk reasoning states through late-layer changes.
}
    \label{fig:main_overview}
\end{figure*}

\section{Introduction}

Large language models (LLMs) often produce confidently wrong answers: they may invent facts that do not exist, or insist on answers to questions they cannot truly solve. Ideally, when the model is unable to answer a question correctly, it should signal so within its response, just as a person who is unsure about something would voice that hesitation rather than confidently guess. If the model can self-assess its reasoning quality at test time and communicate it clearly, a downstream controller can intervene by retrieving evidence, asking a clarifying question, invoking a tool, or abstaining.

A common approach to LLM self-assessment is uncertainty quantification, which estimates how reliable a model's response is after it has been generated~\citep{he2025survey,vashurin2025benchmarking}. Hesitation-like tokens and high-entropy transitions can correlate with internal uncertainty~\citep{wang2025beyond}, and adaptive retrieval systems infer when to intervene from confidence scores, entropy statistics, or response features~\citep{jeong2024adaptive,moskvoretskii2025adaptive,su2403dragin,yao2025seakr}. These signals are useful, but they leave a \emph{visibility problem}: downstream controllers must still infer whether the model knows enough to proceed. A model may internally encode that its reasoning path is fragile while still producing a fluent, confident answer; the goal is for the model to expose these latent warning signals in a legible form before they become confidently wrong outputs. Recent uncertainty-aware training methods improve calibration and reasoning behavior~\citep{li2506confidence,wu2025mitigating,zhao2505learning,guo2025stylebench}, but the central question remains how to make the model's self-assessment explicit and actionable, rather than merely estimated post-hoc. We provide a comprehensive discussion of prior work in Appendix~\ref{app:related_work}.

The real bottleneck, however, is exposure: we need the model to communicate its self-assessment in a form an external controller can act on, at the right moment in generation. This naturally raises a design question: at which point in the response should this signal be exposed? A reliability signal at the end of reasoning supports question-level decisions such as trusting or abstaining, while a signal emitted during reasoning supports mid-trajectory intervention before the answer is committed. We therefore study two complementary post-training forms: verbalizing a confidence score after the response is finalized, and emitting an explicit marker \texttt{<uncertain>} during the reasoning trajectory. This view connects to recent work showing that learned tokens can package complex behavior and provide compact handles for control~\citep{mu2023learning,hewitt2025neologism,hewitt2025we}, as well as to reasoning and retrieval systems that train models to mark actions or reasoning states explicitly~\citep{shao2024deepseekmath,zhang2024backtracking,asai2023selfrag,guo2025meta}. Because external intervention is only worthwhile when its expected benefit exceeds its cost~\citep{liu2024much}, a useful exposure must also be selective. These considerations motivate our central research question:
\begin{center}
    \emph{How should LLMs be trained to expose their reasoning reliability, and what does each design choice imply for the resulting model?}
\end{center}

We study two natural design choices for this exposure. The first is to train the model to verbalize a confidence score for its final answer, indicating how reliable the answer is. The second is to train the model to emit an explicit marker, \texttt{<uncertain>}, during reasoning whenever the current step appears unreliable. Both choices reduce overconfident errors and improve answer quality, and both can serve as triggers for adaptive retrieval. They also leave different signatures inside the model: the end-of-reasoning choice sharpens a confidence-related structure already present in the pretrained model, while the during-reasoning choice reshapes the model's later layers to support an explicit signaling state. The two design choices therefore play complementary rather than competing roles, supporting different downstream decisions and engaging different internal mechanisms.

Our main contributions are as follows.

\begin{itemize}[leftmargin=*]
    \item We frame LLM self-assessment as a problem of \emph{exposure}: training the model to express its reasoning reliability within its own response, rather than estimating it externally after the fact.

    \item We study two natural design choices: training the model to verbalize a confidence score after producing its final answer, and training it to emit an explicit \texttt{<uncertain>} marker during reasoning whenever the current step is unreliable.

    \item Both design choices sharply reduce overconfident errors and improve answer quality, and both can serve as triggers for adaptive retrieval to improve the final response.

    \item Through internal mechanism analysis, we show that the two design choices leave different signatures inside the model: end-of-reasoning verbalization sharpens a confidence-related structure already present in the pretrained model, while during-reasoning signaling reshapes the model's later layers to support an explicit signaling state.
\end{itemize}

\section{Preliminaries}

Self-assessment in language models becomes useful for downstream control only when the model communicates it within its own response. We accordingly study two natural design choices, distinguished by \emph{when} the signal is exposed: an end-of-reasoning signal that summarizes the reliability of the final answer, and a during-reasoning signal that marks high-risk steps before the answer is committed. Given an input question $x$, the model generates a reasoning trajectory $z_{1:T} = (z_1,\dots,z_T)$,
with hidden states $h_t = f_\theta(h_{t-1}, x, z_{<t})$, $t=1,\dots,T$.
The final response induces an answer $\hat{y}$, and we write $Y \in \{0,1\}$ for its correctness indicator. We assume that the hidden trajectory $h_{1:T}$ carries not only task-relevant semantic information but also latent self-assessment information about whether the current reasoning path is reliable. Our goal is to train the model to expose this information explicitly.

The end-of-reasoning signal is a scalar confidence produced after the trajectory is complete: $c = R_{\mathrm{end}}(h_{1:T})$, where $c \in [0,1]$ is intended to summarize final-answer reliability, ideally approximating $\mathbb{P}(Y=1 \mid h_{1:T})$. The during-reasoning signal is a step-level marker emitted while the response is being generated: $a_t = R_{\mathrm{during}}(h_t) \in \{0,1\}$,
where $a_t = 1$ indicates that the model has entered a high-risk reasoning state at step $t$. In our setting, this during-reasoning signal is instantiated by emitting the string \texttt{<uncertain>}.

The two design choices are not interchangeable. An end-of-reasoning signal is a trajectory-level summary: it compresses the reliability of a completed response into a single scalar, which is suitable for selective prediction, abstention, and question-level retrieval gating. However, because reasoning trajectories often contain a single unreliable step surrounded by reliable ones, a scalar score cannot identify which step caused the risk; by the time the score is computed, the answer has already been finalized. A during-reasoning signal addresses this loss of temporal information by surfacing the unreliable step at the moment it arises, before the model commits to an answer. The two design choices therefore support different downstream control decisions and are studied in turn: Section~\ref{sec:verbal-interface} studies the end-of-reasoning signal trained by verbalized confidence, with the trajectory-reweighting interpretation of this objective deferred to Appendix~\ref{app:traj_reweight_proofs}. Section~\ref{sec:local-interface} studies the during-reasoning signal trained by emitting the explicit \texttt{<uncertain>} marker.

\section{Verbalized Confidence}
\label{sec:verbal-interface}

We first study the end-of-reasoning design choice, where the model produces a scalar estimate of the correctness of its final answer. Our goal is not merely to improve calibration metrics, but to understand how such a signal can be learned without degrading the underlying reasoning process. Our central hypothesis is that the pretrained model already contains a weak confidence-related structure in its hidden trajectory but does not express that structure faithfully at the output level. Calibration training should therefore sharpen this existing structure rather than replace the underlying reasoning policy. Concretely, given a reasoning trajectory with hidden states \(h_{1:T}\), the model learns to map this trajectory-level evidence to a confidence \(c\) that better approximates the probability that the final answer is correct.

To test this hypothesis, we train the model with a simple confidence-aware reward:
\(r(x,y,p)=2p-p^2\) if the final answer is correct and \(r(x,y,p)=-p^2\) otherwise. This
directly rewards justified confidence and penalizes overconfident errors, and is
applied only after the full reasoning trajectory is completed. The key intuition is that GRPO should suppress confident wrong trajectories and amplify confident correct ones, improving the model's self-assessment without a separate supervised label. To make this precise, consider a local reweighting view: under a small GRPO-style update, the post-update policy can be approximated as
\begin{equation}
    \pi_{\theta'}(z\mid x)
\;\propto\;
\pi_{\theta}(z\mid x)\exp(\eta\, r(z;x)),\label{approx_dynamics}
\end{equation}
where $z$ is a complete reasoning trajectory for input $x$ and $\eta>0$ is an effective step size. For any two trajectories $z_1,z_2$ for the same question, this implies
\begin{equation}
    \log \frac{\pi_{\theta'}(z_1\mid x)}{\pi_{\theta'}(z_2\mid x)}
=
\log \frac{\pi_{\theta}(z_1\mid x)}{\pi_{\theta}(z_2\mid x)}
+
\eta\bigl(r(z_1;x)-r(z_2;x)\bigr).\label{update_trajectory}
\end{equation}
Higher-confidence errors are thus downweighted more strongly while higher-confidence correct answers are amplified, and the update only redistributes mass among existing trajectories. Appendix~\ref{app:traj_reweight_proofs} formalizes this support-preserving reweighting; Appendix~\ref{app:reward_ablation_transfer} studies alternative reward choices.


On the calibration evaluation, training preserves response quality while sharply improving self-assessment quality: accuracy rises slightly from \(0.345\) to \(0.358\), while ECE drops from \(0.383\) to \(0.049\), Brier score from \(0.504\) to \(0.166\), NLL from \(4.987\) to \(0.498\), and the overconfidence gap from \(+0.523\) to \(+0.045\). The full metric table is deferred to Appendix~\ref{app:verbalized_confidence_details} (Table~\ref{tab:calibration}). More importantly, calibration fundamentally changes the \emph{failure mode} of the model. The baseline is dominated by confidently wrong predictions, whereas the calibrated model assigns substantially lower confidence to incorrect answers. This indicates that calibration does not simply rescale confidence, but suppresses overconfident error without degrading reasoning accuracy. Training dynamics are reported in Appendix~\ref{app:verbalized_confidence_details}, including the reward curve in Figure~\ref{fig:training_curve}; reward-format ablations and Qwen cross-family results are reported in Appendix~\ref{app:reward_ablation_transfer} (Tables~\ref{tab:reward_ablation_compact} and~\ref{tab:cross_family_appendix}).

\paragraph{Mechanism Evidence.}
We complement the headline calibration result with two views of the confidence-token hidden state: a logit lens that aggregates the predicted digits into \textsc{Low}, \textsc{Mid}, and \textsc{High} bins (Figure~\ref{fig:calibration_mechanism_main}), and a PCA projection of the final-layer activations (Figure~\ref{fig:pca_confidence_main}). Both views tell the same story. In the base model, correct and wrong answers both end with dominant mass in the \textsc{High} bin, and the PCA geometry is broad and diffuse. After calibration, low-confidence errors are redirected away from \textsc{High} toward \textsc{Low}, correct answers become more conservative rather than saturating the maximum digit, and the PCA geometry becomes smoother and more ordered along a low-to-high confidence axis. These observations are consistent with calibration acting on a late-stage mapping: the underlying signal already exists in the pretrained model, and training makes its translation into verbalized confidence more selective and cleanly separated. Detailed digit-level routing is deferred to Appendix~\ref{app:verbalized_confidence_details} (Figure~\ref{fig:app_calibration_routing}).

\begin{figure}[t]
    \centering
    \captionsetup[subfigure]{aboveskip=2pt, belowskip=0pt}
    \begin{subfigure}[t]{0.42\linewidth}
        \centering
        \includegraphics[width=\linewidth]{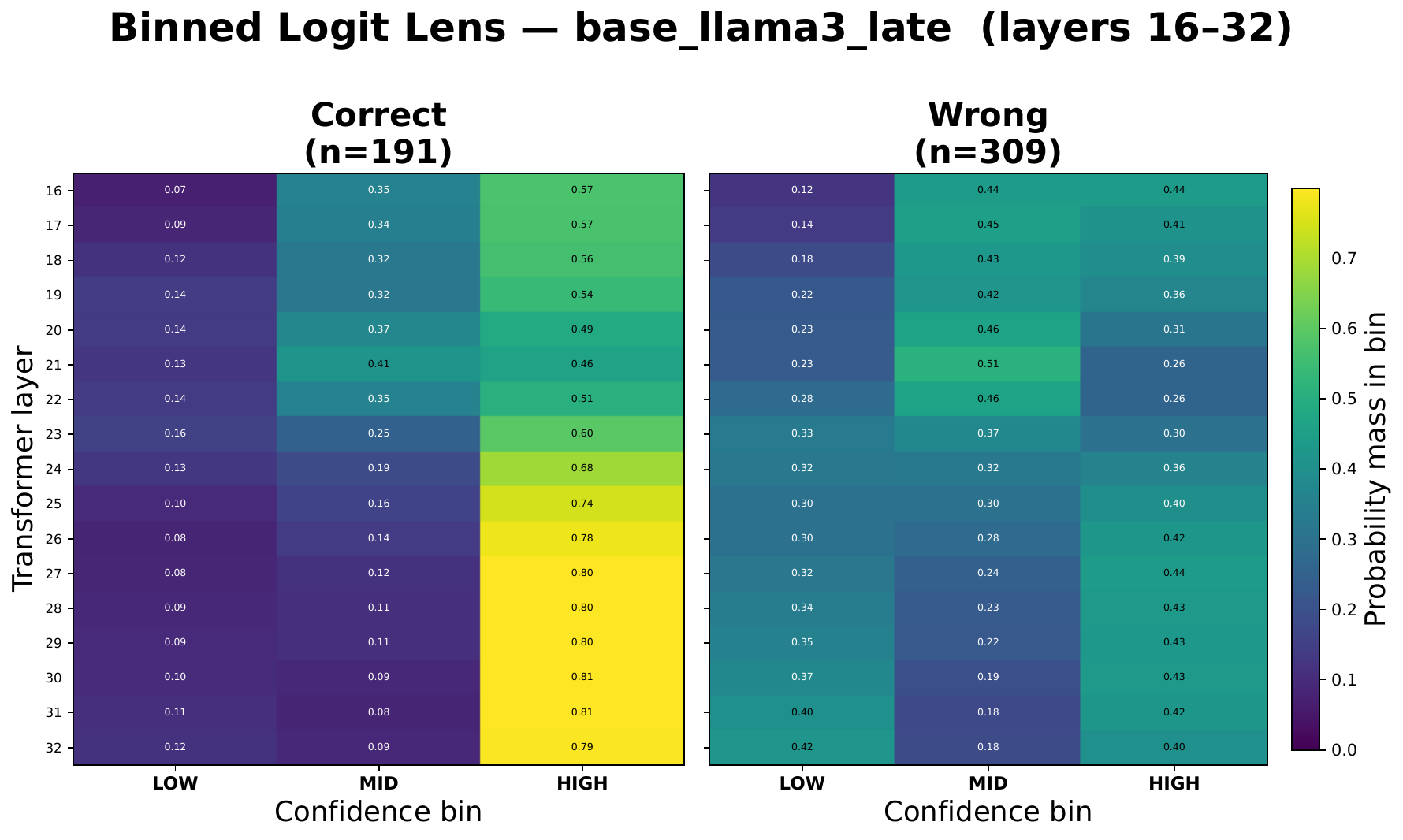}
        \caption{Base}
        \label{fig:heatmap_base}
    \end{subfigure}
    \hfill
    \begin{subfigure}[t]{0.42\linewidth}
        \centering
        \includegraphics[width=\linewidth]{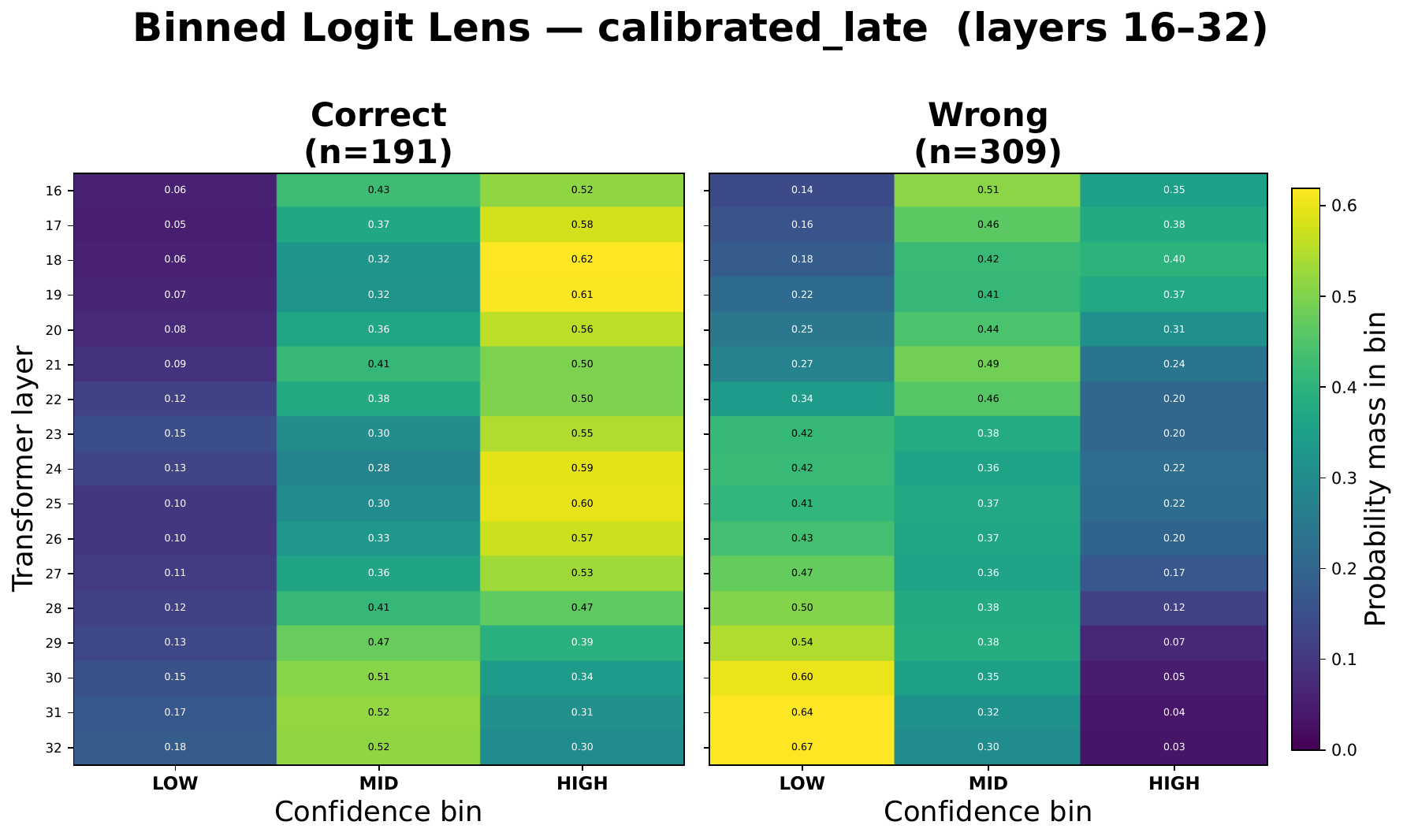}
        \caption{Calibrated}
        \label{fig:heatmap_calibrated}
    \end{subfigure}
    \captionsetup{aboveskip=4pt, belowskip=0pt}
    \caption{Logit-lens analysis of the confidence-token hidden state. Calibration sharpens late-layer confidence routing and yields a cleaner final-layer confidence structure.}
    \label{fig:calibration_mechanism_main}
    \vspace{-0.8em}
\end{figure}

\begin{figure}[t]
    \centering
    \captionsetup[subfigure]{aboveskip=2pt, belowskip=0pt}
    \begin{subfigure}[t]{0.42\linewidth}
        \centering
        \includegraphics[width=\linewidth]{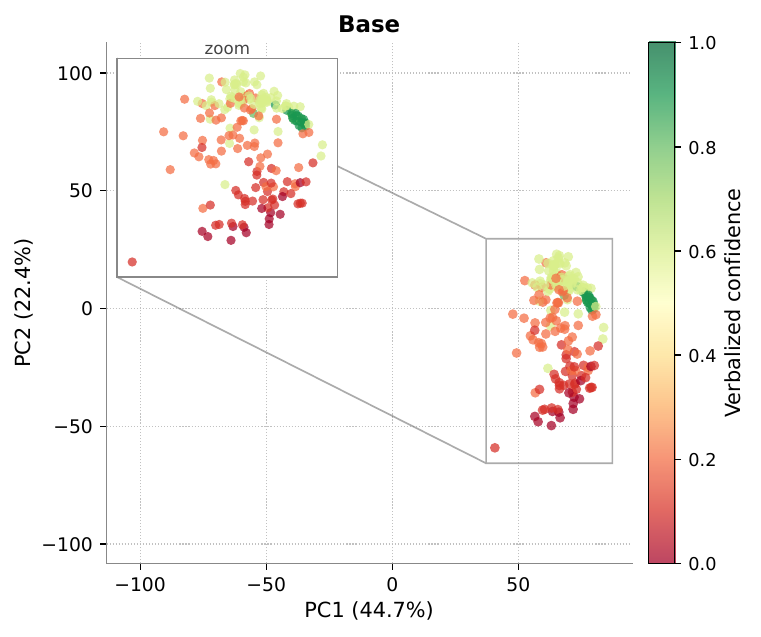}
        \caption{Base model}
        \label{fig:pca_conf_base}
    \end{subfigure}
    \hfill
    \begin{subfigure}[t]{0.42\linewidth}
        \centering
        \includegraphics[width=\linewidth]{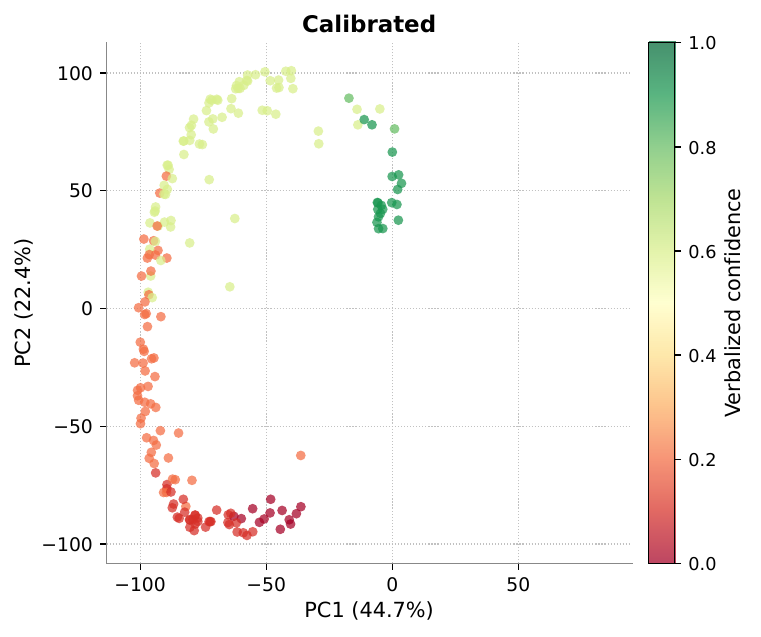}
        \caption{Calibrated model}
        \label{fig:pca_conf_cal}
    \end{subfigure}
    \captionsetup{aboveskip=4pt, belowskip=0pt}
    \caption{PCA analysis of the confidence-token hidden state, colored by verbalized confidence.}
    \label{fig:pca_confidence_main}
    \vspace{-0.8em}
\end{figure}
\paragraph{Error Analysis of the Verbalized Confidence Model.} We next analyze how calibration changes the \emph{type} of errors the model makes. We define \emph{epistemic} errors as wrong answers with confidence above $0.5$, and \emph{aleatoric} errors as wrong answers with confidence at most $0.5$. We also report a stricter epistemic category with confidence above $0.7$, which isolates strongly overconfident hallucinations.

We classify incorrect responses as epistemic or aleatoric using an LLM judge that reads the reasoning text and explicitly ignores the final confidence value; the full prompt is shown in Appendix~\ref{app:judge_error_prompt}.

Table~\ref{tab:vc_error_aggregate} shows the sharpest qualitative shift in the section. In the baseline, almost all errors are epistemic and most are strongly overconfident. After calibration, the majority of errors become low-confidence errors, and the strict epistemic rate drops by more than an order of magnitude. This is the main behavioral conclusion of verbalized confidence: the model changes from being confidently wrong to being uncertain when wrong. Detailed confidence-band, per-dataset conversion, and separation analyses are reported in Appendix~\ref{app:verbalized_confidence_details} (Table~\ref{tab:vc_consistency_and_error_group}).

The same conversion holds across datasets, though its strength varies. The largest reductions occur on MuSiQue and HotpotQA, while Natural Questions remains the hardest case: strongly overconfident errors nearly disappear, but some mistakes remain in the moderate-confidence range.  
\begin{wraptable}[6]{r}{0.25\linewidth}
\vspace{-0.4\baselineskip}
\centering
\scriptsize
\setlength{\tabcolsep}{3.pt}
\begin{tabular}{@{}lcc@{}}
\toprule
\textbf{Error type} & \textbf{Base} & \textbf{Cal.} \\
\midrule
Epistemic & 92.4\% & 34.9\% \\
Aleatoric & 7.6\% & 65.1\% \\
Strict epi. & 88.6\% & 3.9\% \\
\bottomrule
\end{tabular}
\refstepcounter{table}\label{tab:vc_error_aggregate}
\par\vspace{2pt}
\parbox{\linewidth}{\centering\scriptsize\textbf{Table \thetable.} Aggregate error decomposition.}
\vspace{-0.4\baselineskip}
\end{wraptable}
Appendix~\ref{app:verbalized_confidence_details} also shows that calibration increases confidence separation between correct and incorrect answers, rather than simply shifting all scores downward. Overall, these results show that verbalized confidence becomes a calibrated reliability signal by suppressing overconfident errors without materially rewriting the underlying reasoning process.

\section{Reasoning-Time \texttt{<uncertain>} Marker}
\label{sec:local-interface}

The previous section studied self-assessment exposed \emph{after} generation
through verbalized confidence.
We now consider the complementary case in which the model exposes its
self-assessment \emph{during} reasoning.
The goal here is not to estimate the probability that the final answer is
correct, but to mark specific points along the trajectory where the current
reasoning state appears unreliable, at which retrieval or correction
can still change the outcome.

Concretely, we train the model to emit the explicit marker \texttt{<uncertain>}
whenever it encounters such a high-risk state during generation.
This marker is a during-reasoning signal: it does not summarize
final correctness after the fact, but exposes candidate intervention points
before the model has fully committed to an answer.
In Adaptive RAG settings, this is exactly the granularity needed: the
signal arrives in the middle of reasoning, when there is still time to act.

\subsection{\texttt{<uncertain>}-Based Training for Factual Reasoning and
Retrieval Control}

\paragraph{Setup and objective.}
We train the model with GRPO to emit the explicit marker \texttt{<uncertain>} whenever it enters a high-risk reasoning state, while still ending each response with an explicit final answer. We treat each occurrence of this marker in the decoded response as a candidate control point, and a lightweight hidden-state probe decides whether retrieval should actually be triggered.

The training instruction is:
\begin{quote}
\small
You are a helpful reasoning assistant. Think step by step. If at any point
you are uncertain about a fact, emit the special marker \texttt{<uncertain>}
to signal that you need more information. End your response with `Answer:
<your answer>' on the last line.
\end{quote}

Correctness is determined from the final answer line using normalized exact
match, with yes/no matching, date matching, and token-F1 fallback. The reward
is ordered as
\begin{equation}
r(\text{correct, no emit})
>
r(\text{correct, emit})
>
r(\text{wrong, emit})
>
r(\text{wrong, no emit}),
\label{reward_func_uncertain}
\end{equation}
with concrete values \(5.0 > 3.5 > 0.0 > -2.0\)  and an additional repetition penalty when \texttt{<uncertain>} appears more than twice. The key asymmetry is that silent failure is penalized more heavily than uncertain failure, so the model is encouraged to expose likely failure states rather than remain silently overconfident. Unlike the verbalized-confidence objective, which trains a post-hoc summary, this objective acts directly on the reasoning trajectory and is designed to produce an intervention-oriented signal. Reward-design ablations are reported in Appendix~\ref{app:reward_ablation_transfer} (Table~\ref{tab:reward_ablation_compact}).

\newsavebox{\firstemitbox}
\begin{figure}[t]
    \centering
    \sbox{\firstemitbox}{\includegraphics[width=0.48\linewidth]{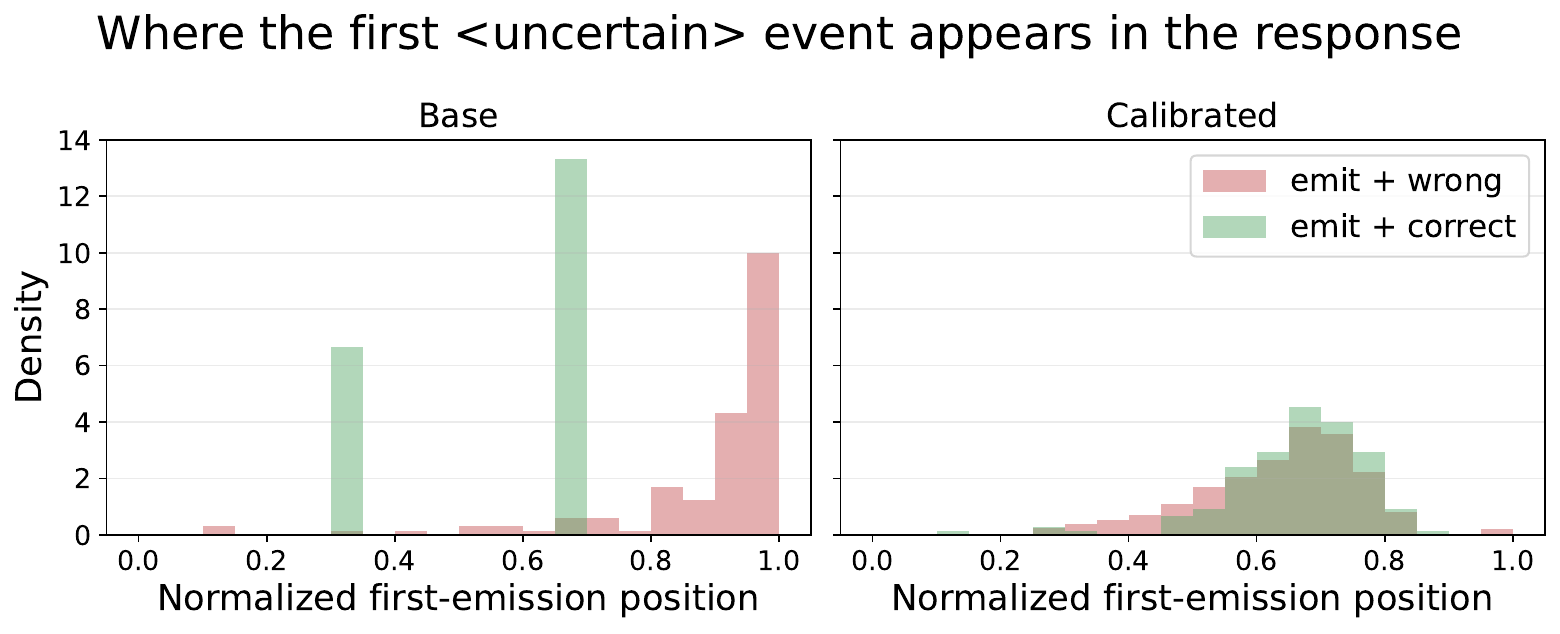}}%
    \begin{minipage}[b]{0.48\linewidth}
        \centering
        \usebox{\firstemitbox}
        \captionof{figure}{First \texttt{<uncertain>} emission position as a fraction of response length.}
        \label{fig:first_emit}
    \end{minipage}
    \hfill
    \begin{minipage}[b]{0.48\linewidth}
        \centering
        \vbox to \ht\firstemitbox{%
            \vfill
            \scriptsize
            \renewcommand{\arraystretch}{1.0}%
            \begin{tabular*}{\linewidth}{@{\extracolsep{\fill}}lcc}
                \toprule
                \textbf{Metric} & \textbf{Base(\%)} & \textbf{Cal(\%)} \\
                \midrule
                Trigger precision & 53.5 & 83.2 \\
                Wrong coverage & 15.1 & 88.2 \\
                Epistemic + emit & 35.0 & 80.1 \\
                Epistemic + no emit & 48.5 & 4.3 \\
                \bottomrule
            \end{tabular*}
            \vfill
        }%
        \captionof{table}{Marker behavior summary.}
        \label{tab:local_behavior_summary}
    \end{minipage}
\end{figure}
Figure~\ref{fig:first_emit} shows where in the response the model first emits \texttt{<uncertain>}. Emissions are distributed across the full range of response positions, not clustered near the end. This confirms that the training objective has successfully instilled mid-reasoning signaling: the model raises the flag while reasoning is still in progress, not after the trajectory has already been committed to. Across six factual reasoning datasets, the calibrated model improves macro-average answer accuracy from \(17.67\%\) to \(28.53\%\), raises answer-line completion from \(58.90\%\) to \(99.93\%\), and increases the fraction of wrong answers that co-occur with \texttt{<uncertain>} emission from \(37.97\%\) to \(58.70\%\). This means the model not only answers more accurately, but also surfaces a much larger share of failures as explicit intervention candidates. Per-dataset breakdowns are reported in Appendix~\ref{app:local_uncertainty_details} (Table~\ref{tab:factual_six_task_compare_combined}); Appendix~\ref{app:reward_ablation_transfer} (Table~\ref{tab:cross_family_appendix}) further shows that the same marker recipe transfers to a Qwen2.5-7B-Instruct model with nearly identical recognized-error rate. Representative four-way examples are shown in Appendix~\ref{app:samples-uncertain}.

\paragraph{Hidden-State Probe for Retrieval Triggering.}
Finally, we test whether the emitted marker exposes a useful internal state
rather than only a surface artifact. A lightweight probe trained on hidden states
around the first \texttt{<uncertain>} emission predicts final-answer wrongness,
with the strongest signal appearing in middle layers. This supports the view
that the marker reveals a structured reasoning-time self-assessment
state that can be used for downstream intervention. We keep the probe as
supporting evidence rather than a separate contribution; the full feature
construction, emitted-subset composition, layer sweep, and probe curve are
reported in Appendix~\ref{app:local_uncertainty_details}
(Figure~\ref{fig:app_probe_layer_results}, Table~\ref{tab:probe_subset_stats},
and Table~\ref{tab:probe_layer_sweep}).

From the perspective of Adaptive RAG, the key quantity is overall wrong-answer coverage on the full dev set, not just probe accuracy inside the emitted subset. Table~\ref{tab:local_behavior_summary} shows that the calibrated pipeline sends \(576/653\) wrong dev answers to retrieval, covering \(88.2\%\) of failures, whereas the base pipeline covers only \(128/848\) (\(15.1\%\)). On the matched test set, a heuristic error-type split shows the same qualitative shift: silent epistemic errors fall from \(48.5\%\) to \(4.3\%\) of wrong answers, while epistemic errors with \texttt{<uncertain>} rise from \(35.0\%\) to \(80.1\%\). Thus, the training does not mainly resolve ambiguity; it turns previously silent factual failures into explicit intervention signals. Overall, the marker functions as a high-recall reasoning-time signal, complementary to verbalized confidence: the verbalized confidence score summarizes final-answer reliability, while \texttt{<uncertain>} exposes points where the model should retrieve or intervene before committing.

\section{Internal Mechanism Analysis}

A central question is whether these gains reflect more than surface-level reward optimization: how can training substantially improve self-assessment quality without degrading reasoning quality? To investigate this, we focus on two complementary analyses: where the training-induced changes concentrate across token positions, and how strongly the model's internal representations are altered. Taken together, these analyses suggest that the self-assessment signal is constructed in a distributed way along the reasoning trajectory and becomes observable only at designated output positions. The two design choices, however, expose this latent signal differently. Verbalized confidence largely preserves representation geometry while sharpening a confidence-related structure already present in the pretrained model. By contrast, the \texttt{<uncertain>} marker induces a broader internal state that reshapes late-layer representations before producing an explicit emission.

\paragraph{Localization: at which positions does training act?}
We compute the token-level KL divergence between base and calibrated model
distributions at every position in the assistant turn, and group positions by
their semantic type (confidence digit, structural label, reasoning token,
\texttt{<uncertain>} position, nearby context).
This directly reveals which positions absorb the distributional change.

\begin{figure*}[t]
    \centering
    \captionsetup[subfigure]{aboveskip=2pt, belowskip=0pt}
    \begin{subfigure}[b]{0.48\textwidth}
        \centering
        \includegraphics[width=\linewidth]{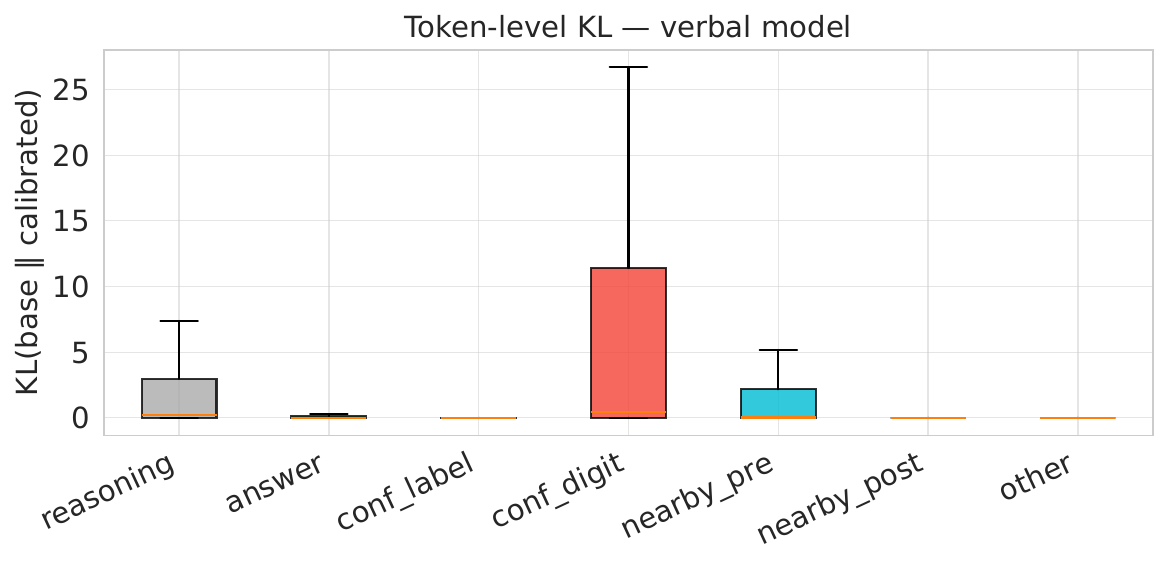}
        \caption{Verbalized confidence}
        \label{fig:main-exp1-verbal}
    \end{subfigure}
    \hfill
    \begin{subfigure}[b]{0.48\textwidth}
        \centering
\includegraphics[width=\linewidth]{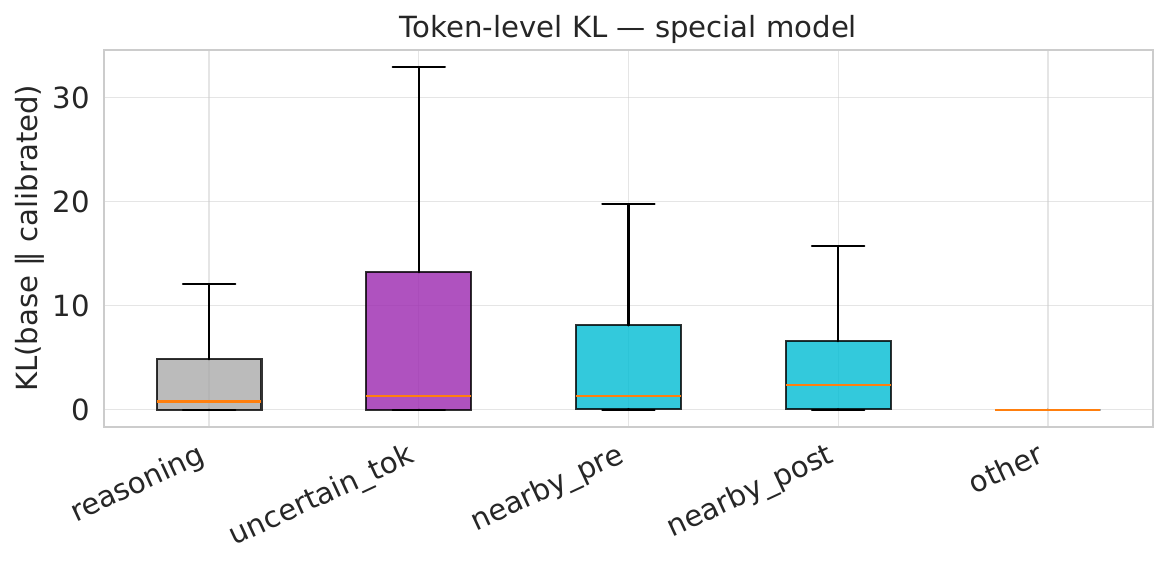}
        \caption{\texttt{<uncertain>} marker}
        \label{fig:main-exp1-special}
    \end{subfigure}
    \captionsetup{aboveskip=4pt, belowskip=0pt}
    \caption{Token-level KL by position type. Both objectives concentrate distributional change at their signal positions, but the \texttt{<uncertain>} marker has a broader local footprint.}
    \label{fig:main-localization}
    \vspace{-0.8em}
\end{figure*}

Figure~\ref{fig:main-localization} shows that both training objectives
successfully localize their effect at the intended output position.
Verbalized confidence produces a point-like signature: only the digit
token is changed, leaving the surrounding format and all reasoning tokens
largely unaffected.
The \texttt{<uncertain>} marker produces a wider footprint: KL is elevated
not just at the emission token but in the tokens immediately surrounding it,
indicating that the explicit signal is preceded by a change in the model's
local computation state.
Localization is therefore a property of both design choices. Additional
hidden-state results, reported in Appendix~\ref{app:patching},
provide supporting evidence that the signal position is better interpreted as
an exposure point than as a self-contained causal circuit. We treat those
results as suggestive rather than definitive, since the intervention changes
only a single token state.

\paragraph{Representation geometry: how deeply does training rewrite
the model?}
We measure this using Centered Kernel Alignment (CKA), which compares the
geometry of hidden representations at signal-token positions between the base
and calibrated model, layer by layer.
A CKA value of $1.0$ means the representations are geometrically identical;
values below $1.0$ indicate structural divergence.

\begin{figure*}[t]
    \centering
    \captionsetup[subfigure]{aboveskip=2pt, belowskip=0pt}
    \begin{subfigure}[b]{0.48\textwidth}
        \centering
        \includegraphics[width=\linewidth]{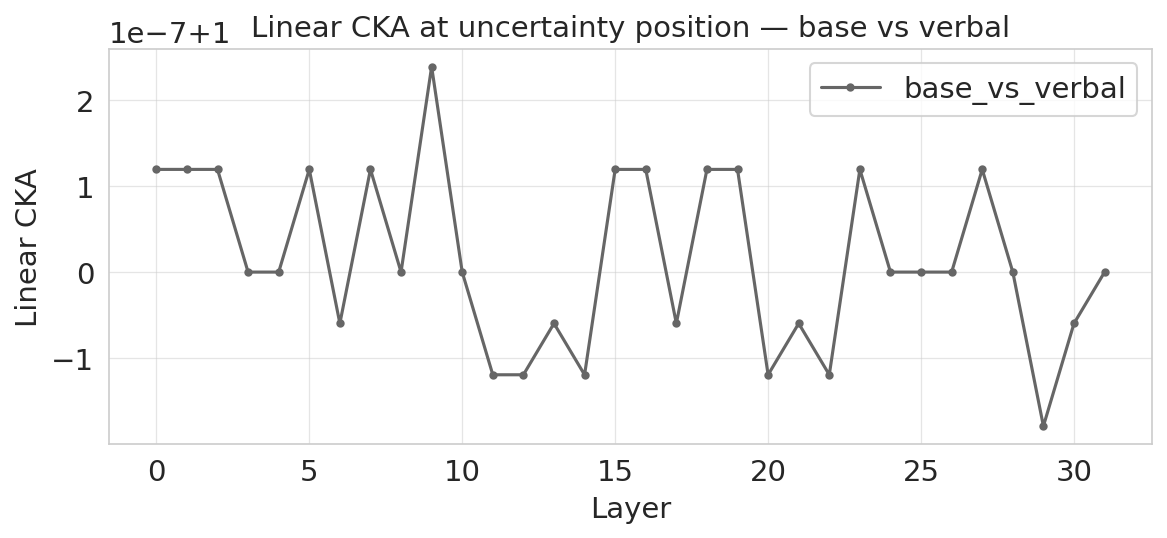}
        \caption{Verbalized confidence}
        \label{fig:main-exp3-verbal}
    \end{subfigure}
    \hfill
    \begin{subfigure}[b]{0.48\textwidth}
        \centering
        \includegraphics[width=\linewidth]{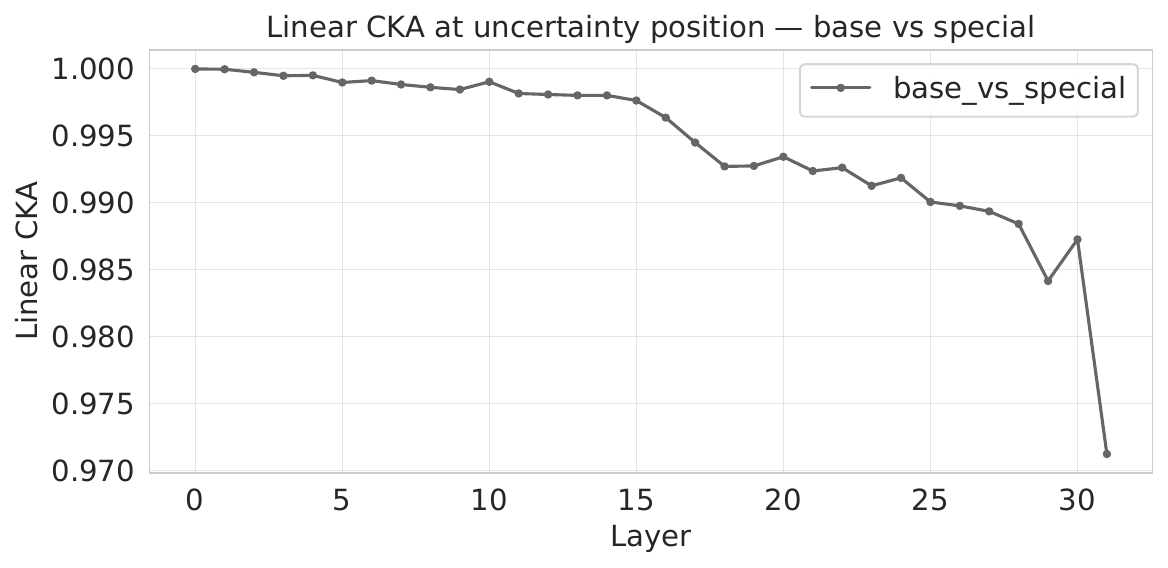}
        \caption{\texttt{<uncertain>} marker}
        \label{fig:main-exp3-special}
    \end{subfigure}
    \captionsetup{aboveskip=4pt, belowskip=0pt}
    \caption{Layer-wise CKA between base and calibrated models. Verbalized confidence preserves representation geometry, whereas the \texttt{<uncertain>} marker induces increasing late-layer divergence.}
    \label{fig:main-cka}
    \vspace{-0.8em}
\end{figure*}

Figure~\ref{fig:main-cka} provides the clearest contrast between the two
design choices.
Verbalized confidence achieves a large improvement in calibration quality while leaving the representation geometry nearly unchanged under this CKA diagnostic: the CKA curve remains close to 1.0 from input to output layer. This suggests that the model did not need to build a substantially new representation from scratch; rather, calibration sharpens and organizes an existing confidence-related geometry on top of the pretrained representation.
The \texttt{<uncertain>} marker takes a different path: late-layer
representations diverge progressively, indicating that explicit mid-reasoning
emission requires the model to actively build a new internal state, not just
refine an existing output mapping.

An important implication is that raw parameter movement is not sufficient to
explain behavioral interference.
Appendix~\ref{app:drift} shows that the two calibrated models exhibit broadly
similar parameter-space drift patterns, concentrated in attention
\texttt{v\_proj}/\texttt{o\_proj} and MLP projections, with little drift in
LayerNorm terms (Figure~\ref{fig:app-exp4-drift}).
Yet these similarly sized and similarly located updates have sharply different
representation-level consequences: verbalized confidence preserves geometry,
whereas the \texttt{<uncertain>} marker rewrites late-layer states.
The key distinction between the two design choices is therefore not how much they
update the model, but whether the objective can be realized by sharpening an
existing confidence-related structure or instead requires building a new
internal state for explicit signaling. Additional patching and per-example
linkage analyses are deferred to Appendix~\ref{app:patching} and
Appendix~\ref{app:linkage}.

\section{Evaluation}
We evaluate on five factual QA benchmarks spanning multi-hop reasoning and open-domain recall: HotpotQA~\cite{yang2018hotpotqa}, MuSiQue~\cite{trivedi2022musique}, 2WikiMultihopQA~\cite{ho2020constructing}, Natural Questions~\cite{kwiatkowski2019natural}, and TriviaQA~\cite{joshi2017triviaqa}. The goal is not only to improve answer quality, but to test whether trained self-assessment signals outperform simpler recalibration, detection, and retrieval-control alternatives.
\subsection{Calibration Evaluation}
Sections~\ref{sec:verbal-interface} and~\ref{sec:local-interface} showed the native effects of the two design choices. We now ask whether those gains can be explained by simpler alternatives; implementation details for all baselines are in Appendix~\ref{app:baseline_impl_details}.

In Panel~A, \textbf{P(True)}~\cite{kadavath2022language} re-queries the model with a binary correctness prompt; \textbf{Global TS}~\cite{guo2017calibration} and \textbf{ATS}~\cite{xie2024calibrating} are post-hoc temperature scaling (single or input-dependent) on the base model's confidences; \textbf{SFT-Conf}~\cite{kapoor2024large} and \textbf{SFT-KWDK}~\cite{luo2025your} supervised-fine-tune the model to reproduce a continuous F1-derived target or a four-bucket label. In Panel~B, \textbf{Emit heur.} prompts the untrained base to emit \texttt{<uncertain>}; \textbf{Hidden probe} and \textbf{Output clf.} are passive wrongness detectors using base-model hidden states or surface response features; \textbf{Self-RAG}~\cite{asai2023selfrag}, \textbf{FLARE}~\cite{jiang2023flare}, and \textbf{ADARAGUE}~\cite{moskvoretskii2025adaptive} are retrieval-controller analogs whose retrieval signals we map to binary triggers.

\begin{table*}[t]
\centering
\scriptsize
\renewcommand{\arraystretch}{0.97}
\setlength{\tabcolsep}{3.5pt}
\begin{tabular}{lccccc|lccccc}
\toprule
\multicolumn{6}{c|}{\textbf{A. Verbalized confidence}} &
\multicolumn{6}{c}{\textbf{B. \texttt{<uncertain>} marker}} \\
\midrule
Method & EM & F1 & Brier & ECE & OConf &
Method & Emit & Prec. & Recall & Acc$_{\neg t}$ & Wrong/Pos. \\
\midrule
Base & 24.5 & 37.3 & -0.108 & 0.357 & 88.5 &
Emit heur. & 0.336 & \textbf{0.959} & 0.444 & 0.392 & 0.726 \\
P(True) & 24.4 & 37.2 & -0.096 & 0.340 & 39.7 &
Hidden probe & 0.699 & 0.889 & 0.856 & 0.653 & 0.726 \\
Global TS & 24.5 & 37.3 & +0.116 & 0.185 & 69.4 &
Output clf. & 0.925 & 0.754 & \textbf{0.961} & 0.622 & 0.726 \\
ATS & 24.5 & 37.3 & +0.123 & 0.166 & 53.4 &
Self-RAG & 0.444 & 0.861 & 0.478 & 0.250 & 0.799 \\
SFT-Conf & 21.1 & 33.5 & +0.083 & 0.226 & 7.3 &
FLARE & 0.586 & 0.738 & 0.598 & 0.300 & 0.722 \\
SFT-KWDK & 22.4 & 34.5 & +0.105 & 0.204 & 8.6 &
ADARAGUE & 0.527 & 0.216 & 0.687 & 0.690 & \textbf{0.166} \\
\textbf{Ours} & \textbf{27.4} & \textbf{38.2} & \textbf{+0.210} & \textbf{0.036} & \textbf{3.2} &
\textbf{Ours} & 0.592 & 0.799 & 0.883 & \textbf{0.719} & 0.528 \\
\bottomrule
\end{tabular}
\caption{Panel A evaluates verbalized confidence; OConf is the percentage of wrong answers with confidence \(>0.5\). Panel B evaluates \texttt{<uncertain>} marker triggers; Prec. = P(wrong \(\mid\) trigger), Recall = P(trigger \(\mid\) wrong), and Acc$_{\neg t}$ is accuracy on untriggered examples.}
\label{tab:combined_baselines}
\end{table*}

\paragraph{Verbalized confidence.}
Three patterns emerge from Panel~A. First, prompting-based self-evaluation (P(True)) reduces overconfident wrong answers (OConf $88.5 \to 39.7$) but barely changes calibration (ECE $0.357 \to 0.340$), suggesting that self-querying alone redistributes confidence rather than aligning it with correctness. Second, post-hoc temperature scaling (Global TS, ATS) sharply improves ECE (down to $0.166$--$0.185$) but leaves a majority of wrong answers still over-confident (OConf $53.4$--$69.4$); temperature simply tilts the entire confidence distribution rather than fixing the overconfident-error pattern. Third, supervised fine-tuning (SFT-Conf, SFT-KWDK) does suppress overconfidence (OConf $7.3$--$8.6$), but at the cost of answer accuracy (EM drops from $24.5$ to $21.1$--$22.4$) and without matching ATS on ECE. Our GRPO-trained verbalized confidence is the only method that simultaneously achieves the lowest ECE ($0.036$), the lowest OConf ($3.2$), and the highest EM ($27.4$). The gains from verbalized confidence are thus not reducible to post-hoc rescaling, self-evaluation prompting, or supervised relabeling.

\paragraph{\texttt{<uncertain>} marker.}
This comparison asks whether wrongness is surfaced early enough for intervention. Passive detectors and retrieval-controller baselines show that failure is partially detectable without training, but the trained marker achieves the best untouched-set accuracy while remaining competitive on precision and recall. This supports the main claim from Section~\ref{sec:local-interface}: training changes the generator so that more failures become explicit control signals, rather than merely attaching a detector after generation.

\subsection{Downstream Task Performance: Adaptive RAG Triggering}

\newcommand{\ragbest}[1]{\cellcolor{green!22}\textbf{#1}}
\newcommand{\ragsecond}[1]{\cellcolor{green!13}#1}
\newcommand{\ragthird}[1]{\cellcolor{green!7}#1}
\newcommand{\ragbound}[1]{\cellcolor{gray!18}\textit{#1}}
\newcommand{\trigbest}[1]{\cellcolor{yellow!22}\textbf{#1}}
\newcommand{\ragglobal}[1]{\cellcolor{blue!8}#1}
\newcommand{\raglocal}[1]{\cellcolor{orange!10}#1}

We evaluate downstream retrieval control on \(500\) questions from each of the five QA datasets. Each method first answers without retrieval, optionally triggers one retrieval step, and then answers again using the retrieved evidence. Table~\ref{tab:adaptive_rag_results} compares our two trained methods against no-retrieval (\texttt{No-Ret}), always-retrieval (\texttt{Ret-All}), Self-RAG (\texttt{SR-7B}/\texttt{SR-13B})~\cite{asai2023selfrag}, ADARAGUE~\cite{moskvoretskii2025adaptive}, FLARE~\cite{jiang2023flare}, DRAGIN~\cite{su2403dragin}, and prompt-only base controls. \texttt{Base-Verbal} and \texttt{Base-UncTok} use the untrained base model with the verbalized-confidence and \texttt{<uncertain>} marker prompts, respectively; \texttt{Verbal-Calibrate} and \texttt{Uncertain-Calibrate} are our two trained methods.

\begin{table*}[t]
\centering
\caption{
Adaptive RAG evaluation results. EM, F1 (both \%), and trigger rate T (\%) are shown per dataset when available. 
Darker shaded cells indicate stronger EM/F1 performance within each metric column; yellow marks the highest selective trigger rate, excluding \textit{Ret-All}.
}
\label{tab:adaptive_rag_results}
\resizebox{\linewidth}{!}{%
\begin{tabular}{l ccc ccc ccc ccc ccc ccc}
\toprule
& \multicolumn{3}{c}{HotpotQA}
& \multicolumn{3}{c}{MuSiQue}
& \multicolumn{3}{c}{2WikiMultiHop}
& \multicolumn{3}{c}{NQ}
& \multicolumn{3}{c}{TriviaQA}
& \multicolumn{3}{c}{Overall} \\
\cmidrule(lr){2-4}\cmidrule(lr){5-7}\cmidrule(lr){8-10}
\cmidrule(lr){11-13}\cmidrule(lr){14-16}\cmidrule(lr){17-19}
Method & EM & F1 & T & EM & F1 & T & EM & F1 & T & EM & F1 & T & EM & F1 & T & EM & F1 & T \\
\midrule
No-Ret
& 23.4 & 31.9 & --
& 5.6 & 10.3 & --
& 16.4 & 20.1 & --
& 29.2 & \ragthird{42.0} & --
& \ragthird{54.8} & \ragthird{61.2} & --
& 25.9 & 33.1 & -- \\

SR-7B
& 4.2 & 16.6 & 57.8
& 0.6 & 5.2 & 56.0
& 4.4 & 14.8 & 45.6
& 17.8 & 22.0 & 18.0
& 5.6 & 24.9 & 53.6
& 6.5 & 16.7 & 46.2 \\

SR-13B
& 2.6 & 16.0 & 42.4
& 0.8 & 6.1 & 39.8
& 3.4 & 16.6 & 30.0
& \ragthird{30.8} & 36.9 & 5.4
& 6.4 & 36.5 & 29.8
& 8.8 & 22.4 & 29.5 \\

\midrule
\ragbound{Ret-All}
& \ragbound{34.0} & \ragbound{44.9} & \ragbound{100}
& \ragbound{9.8} & \ragbound{17.4} & \ragbound{100}
& \ragbound{30.0} & \ragbound{34.7} & \ragbound{100}
& \ragbound{25.4} & \ragbound{35.8} & \ragbound{100}
& \ragbound{43.0} & \ragbound{49.1} & \ragbound{100}
& \ragbound{28.4} & \ragbound{36.4} & \ragbound{100} \\

ADARAGUE
& 27.8 & 37.2 & 57.0
& 9.6 & 15.9 & 53.8
& 20.6 & 25.8 & 57.4
& 28.6 & 40.0 & 21.8
& 52.2 & 58.1 & 29.0
& 27.8 & 35.4 & 43.8 \\

FLARE
& 21.8 & 35.2 & \trigbest{99.2}
& 6.4 & 13.8 & \trigbest{99.8}
& 15.2 & 22.6 & \trigbest{99.4}
& 19.8 & 31.5 & \trigbest{99.0}
& 41.0 & 49.9 & \trigbest{97.4}
& 20.8 & 30.6 & \trigbest{99.0} \\

DRAGIN
& \ragthird{34.6} & \ragthird{48.4} & 87.6
& \ragthird{17.0} & \ragsecond{27.1} & 87.0
& \ragthird{34.4} & \ragbest{43.9} & 85.2
& 23.8 & 35.9 & 70.4
& \ragsecond{53.6} & \ragsecond{60.6} & 54.0
& \ragthird{32.7} & \ragthird{43.2} & 76.8 \\

Base-Verbal
& 21.2 & 28.8 & 15.2
& 10.8 & 17.5 & 13.6
& 17.0 & 19.7 & 9.2
& 17.8 & 29.7 & 22.8
& 42.4 & 43.8 & 4.6
& 21.8 & 27.9 & 13.1 \\

Base-UncTok
& 22.5 & 33.1 & 3.5
& 8.8 & 16.2 & 4.3
& 17.7 & 21.8 & 2.6
& 18.8 & 28.1 & 3.7
& 34.8 & 36.8 & 4.5
& 20.5 & 27.2 & 3.5 \\

\ragglobal{Verbal-Calibrate}
& \ragsecond{42.0} & \ragbest{52.8} & 61.6
& \ragbest{21.8} & \ragbest{28.8} & 76.8
& \ragbest{38.4} & \ragsecond{42.9} & 48.2
& \ragbest{52.4} & \ragbest{54.4} & 25.0
& \ragsecond{63.2} & \ragsecond{72.5} & 28.8
& \ragbest{41.6} & \ragbest{50.5} & 48.1 \\

\raglocal{Uncertain-Calibrate}
& \ragbest{42.6} & \ragsecond{52.7} & 67.4
& \ragsecond{17.6} & \ragthird{24.1} & 94.2
& \ragsecond{36.2} & \ragthird{39.6} & 59.2
& \ragsecond{41.4} & \ragsecond{51.0} & 52.0
& \ragbest{66.6} & \ragbest{73.2} & 34.0
& \ragsecond{40.9} & \ragsecond{48.1} & 61.4 \\
\bottomrule
\end{tabular}%
}
\end{table*}

\paragraph{Main result.} In Table~\ref{tab:adaptive_rag_results}, darker shaded cells indicate stronger performance within each metric column, blue shades are used for the verbalized-confidence evaluation, and orange shades are used for the \texttt{<uncertain>} marker evaluation. Yellow shading marks the highest trigger/emission rate, which reflects intervention frequency rather than necessarily better performance.
Both trained methods outperform non-adaptive and retrieval-heavy baselines. \texttt{Verbal-Calibrate} achieves the best overall result (\(41.6\%\) EM, \(50.5\%\) F1) with a \(48.1\%\) trigger rate, while \texttt{Uncertain-Calibrate} reaches \(40.9\%\) EM and \(48.1\%\) F1 with a higher \(61.4\%\) trigger rate. The gains over FLARE and DRAGIN are especially informative because those methods retrieve far more often; the improvement therefore comes from better control signals, not simply more retrieval. The weak \texttt{Base-Verbal} and \texttt{Base-UncTok} results show the same point from the other side: exposing a marker or confidence score is insufficient unless the model is trained to use it.

\paragraph{Roles of the two methods.}
The two methods behave as intended. Verbalized confidence is stronger and more retrieval-efficient overall, making it a better question-level gate. The \texttt{<uncertain>} marker retrieves more aggressively and performs best on some datasets, consistent with a high-recall intervention signal during reasoning. This supports the paper's central framing: end-of-reasoning self-assessment is useful for deciding whether to trust a completed answer, while during-reasoning self-assessment is useful for deciding when to intervene before the model commits.

\section{Conclusion}
We studied LLM self-assessment as a problem of exposure: rather than estimating reliability after generation, we train the model to express it within its own response, in a form a downstream controller can act on. Within a unified post-training framework, we studied two design choices that differ in when the signal is exposed: verbalizing a confidence score after the final answer, and emitting an  \texttt{<uncertain>} marker during reasoning. The two design choices produce different but complementary benefits: verbalized confidence is most effective for final-answer trust and retrieval gating, while the \texttt{<uncertain>} marker is most effective for exposing silent failures early enough for intervention.

Our results also show that these gains are not merely formatting effects. Verbalized-confidence training sharpens a weak confidence-related structure already present in the pretrained model, whereas \texttt{<uncertain>} training induces a broader late-layer state that supports explicit mid-reasoning signaling. Together, these findings suggest that effective self-assessment in LLMs should be trained as task-matched communication: an end-of-reasoning confidence summary when the decision is whether to trust the final answer, and a during-reasoning marker when the decision is whether the model needs intervention before it fully commits.

\paragraph{Limitations.}
We study factual
QA and adaptive retrieval with a single during-reasoning marker; coding and agentic multi-turn settings may require richer
feedback and multiple specialized markers for different failure modes or tool
calls. 

\bibliographystyle{plain}
\bibliography{refs}

\clearpage 

\appendix

\begin{center}
    {\Large \textbf{Appendix}}
\end{center}




\startcontents[appendix]


\printcontents[appendix]{}{1}{\setcounter{tocdepth}{3}}

\clearpage

\section{Related Work}
\label{app:related_work}

\paragraph{Uncertainty estimation and calibration in language models.}
A large body of work studies how to estimate whether a model's generated answer is reliable.
Classical calibration methods evaluate whether predicted probabilities match empirical correctness, with temperature scaling serving as a simple and widely used post-hoc recalibration method~\citep{guo2017calibration}.
For language models, uncertainty is harder to define because outputs are free-form sequences rather than fixed-class predictions.
Recent surveys and benchmarks organize LLM uncertainty estimation into likelihood-based, sampling-based, semantic, verbalized, and hybrid methods~\citep{he2025survey,vashurin2025benchmarking}.
Semantic uncertainty methods aggregate generations by meaning rather than surface form, showing that uncertainty in natural language should often be measured over semantic equivalence classes rather than exact strings~\citep{kuhn2023semantic}.
Other work develops uncertainty quantification for black-box LLMs~\citep{lin2023generating}, evaluates rank-calibration in language models~\citep{huang2024uncertainty}, and studies subjective uncertainty and calibration in natural language generation~\citep{wang2024subjective}.
Our work is complementary to this line: rather than only post-hoc estimation, we train the model to expose its self-assessment explicitly within the response itself.

\paragraph{Verbalized confidence and model self-knowledge.}
Several papers ask whether language models know when they are likely to be correct.
Early work showed that models can be trained to express uncertainty in words and that such expressions can become more calibrated than raw model behavior~\citep{lin2022teaching}.
Related work studies whether models can judge the truth of their own answers through P(True)-style prompting and self-evaluation~\citep{kadavath2022language}.
Empirical studies of confidence elicitation show that LLMs can produce verbalized confidence scores, but these scores are often sensitive to prompting and may remain overconfident without additional training~\citep{xiong2023can,yang2024verbalized,yona2024can}.
Other recent work argues that pretrained LLMs already contain useful confidence signals that can be extracted or recalibrated without full retraining~\citep{luo2025your}, while supervised approaches teach models to better distinguish what they know from what they do not know~\citep{kapoor2024large}.
Recent studies also investigate how verbalized confidence affects generation diversity and self-verification behavior~\citep{zhang2025verbalized,jang2025verbalized,fu2025deep}.
Our verbalized-confidence training builds on this direction, but differs in two ways: we train confidence with a proper-scoring-rule-style reward rather than only eliciting it by prompting, and we analyze how calibration changes the model's hidden confidence structure.

\paragraph{Reward-based post-training for calibrated behavior.}
Post-training with reinforcement learning has become a standard way to shape LLM behavior~\citep{ouyang2022training,shao2024deepseekmath,guo2025deepseek}.
Recent work shows that reward design can strongly affect reasoning style, calibration, and hallucination behavior~\citep{li2506confidence,wu2025mitigating,leng2024taming}.
In particular, calibration-aware rewards and confidence-aware RL objectives encourage models to assign lower confidence to incorrect answers and higher confidence to correct ones.
Our trajectory-reweighting analysis gives a local mathematical account of this effect: under a small policy-improvement step, the reward tilts probability mass away from overconfident wrong trajectories and toward confident correct trajectories.
This perspective connects reward shaping to a support-preserving redistribution of existing reasoning trajectories, rather than treating calibration gains as purely post-hoc rescaling.
It also relates to recent analyses of reinforcement learning for reasoning, which study how RL changes token-level behavior and trajectory selection~\citep{wang2025beyond,zhao2505learning,zhu2025path}.

\paragraph{Selective prediction, abstention, and retrieval decisions.}
Uncertainty estimates are useful only when connected to a downstream decision, such as abstaining, asking for help, or retrieving evidence.
Adaptive retrieval systems use uncertainty or complexity estimates to decide when retrieval is worthwhile.
Adaptive-RAG learns to adapt retrieval behavior based on question complexity~\citep{jeong2024adaptive}; FLARE triggers retrieval during generation based on low-confidence token predictions~\citep{jiang2023flare}; DRAGIN retrieves based on real-time information needs~\citep{su2403dragin}; and ADARAGUE reintroduces uncertainty features into adaptive retrieval control~\citep{moskvoretskii2025adaptive}.
SEAKR similarly studies self-aware knowledge retrieval for adaptive RAG~\citep{yao2025seakr}.
Other work studies how much retrieval can improve reasoning~\citep{liu2024much}, improves retrievers for reasoning-intensive tasks~\citep{shao2025reasonir}, and shows that simple retrieval can help challenging reasoning benchmarks~\citep{lyu2025frustratingly}.
Probing-RAG uses model-internal probing to guide selective retrieval~\citep{baek2025probing}, which is especially close to our hidden-state probe analysis.
Our work shares the goal of reducing unnecessary retrieval while improving final answers, but differs in where the signal comes from: instead of relying only on external uncertainty features or passive detectors, we train the generator itself to emit a confidence score or a reasoning-time \texttt{<uncertain>} marker that can be used as the trigger.

\paragraph{Learned tokens, control markers, and reasoning-time intervention.}
A related line of work studies whether special tokens or learned markers can package complex behaviors into compact controllable symbols.
Gist tokens compress prompts into learned latent handles~\citep{mu2023learning}, and recent work on neologism learning studies how new tokens can acquire controllable meanings through training~\citep{hewitt2025neologism,hewitt2025we}.
In reasoning and tool-use settings, models can also learn to emit action-like markers that control retrieval, critique, or generation behavior.
Self-RAG trains models to retrieve, generate, and critique using reflection tokens~\citep{asai2023selfrag}, while backtracking tokens allow models to mark unsafe or undesirable generation paths and revise them~\citep{zhang2024backtracking}.
Confidence-token routing further shows that explicit learned tokens can support model selection or rejection decisions~\citep{chuang2024learning}.
Our \texttt{<uncertain>} marker is closest in spirit to these learned control-token methods, but its role is specifically to expose a high-risk reasoning state before the answer is finalized.
This makes it different from final-answer confidence: the marker is step-level, binary, and intervention-oriented.

\paragraph{Internal states and mechanistic evidence for self-assessment.}
Several works suggest that a model's hidden states can contain information about truthfulness, correctness, or hallucination risk that is not always faithfully expressed in the output.
Latent-knowledge work shows that internal representations can encode truth-related information without direct supervision~\citep{burns2022discovering}.
Other studies find that hidden states can reveal when a model is lying or likely to hallucinate~\citep{azaria2023internal,ji2024llm}, and semantic-entropy probes show that hallucination risk can sometimes be detected from internal activations~\citep{kossen2024semantic}.
Mechanistic and layerwise analyses of post-training further suggest that fine-tuning can affect different layers and token positions unevenly~\citep{wang2025beyond}.
Our analysis follows this motivation but focuses on comparing two trained self-assessment signals.
We find that verbalized confidence largely preserves representation geometry while sharpening a confidence-related structure, whereas the \texttt{<uncertain>} marker produces broader late-layer changes around the emission event.
This supports the central distinction of the paper: end-of-reasoning confidence summarizes reliability after the trajectory is complete, while during-reasoning marking exposes a risky step before the model commits.

\paragraph{Positioning of this work.}
Overall, prior work provides strong tools for estimating uncertainty, calibrating confidence, triggering retrieval, and learning control tokens.
However, these directions are often studied separately: uncertainty estimation is usually post-hoc, retrieval controllers often rely on external features or passive detectors, and learned control tokens are not necessarily trained to represent reasoning-time uncertainty.
Our work connects these threads by asking where self-assessment should be exposed within the response.
We study two complementary choices under a unified post-training view: verbalized confidence after the final answer, which supports trust, abstention, and question-level retrieval gating; and a \texttt{<uncertain>} marker during reasoning, which supports intervention before commitment.
The empirical and mechanistic results show that these two choices are not interchangeable output formats, but distinct ways of making the model's self-assessment available to downstream systems at different moments in generation.

\section{Proofs for the trajectory-reweighting analysis}
\label{app:traj_reweight_proofs}

In this appendix, we give short proofs for the theoretical claims in Section~3. Throughout, we use the tilted-distribution idealization
\begin{equation}
\label{eq:app-tilted-policy}
\pi_{\theta'}(z \mid x)
\;\propto\;
\pi_\theta(z \mid x)\exp\!\bigl(\eta\, r(z;x)\bigr),
\end{equation}
as a first-order analytical model of one-step uncertainty-aware policy improvement under GRPO, where $\eta>0$ is an effective step size.

\paragraph{Normalization form.}
For fixed input $x$, define the partition function
\begin{equation}
\label{eq:app-partition-function}
Z(x) \;=\; \sum_{z} \pi_\theta(z \mid x)\exp\!\bigl(\eta\, r(z;x)\bigr).
\end{equation}
Then the reweighted policy can be written as
\begin{equation}
\label{eq:app-normalized-policy}
\pi_{\theta'}(z \mid x)
=
\frac{\pi_\theta(z \mid x)\exp\!\bigl(\eta\, r(z;x)\bigr)}{Z(x)}.
\end{equation}

\begin{proposition}[One-step relative improvement under uncertainty-aware RL]
\label{prop:one-step-relative-improvement}
For any two trajectories $z_1, z_2$ for the same input $x$, the reweighted policy satisfies
\begin{equation}
\label{eq:app-log-update}
\log \frac{\pi_{\theta'}(z_1 \mid x)}{\pi_{\theta'}(z_2 \mid x)}
=
\log \frac{\pi_\theta(z_1 \mid x)}{\pi_\theta(z_2 \mid x)}
+
\eta \bigl(r(z_1;x)-r(z_2;x)\bigr).
\end{equation}
Hence, a single RL improvement step increases the relative likelihood of higher-reward trajectories and decreases that of lower-reward ones.
\end{proposition}

\begin{corollary}[Selective suppression of overconfident errors]
\label{cor:app-selective-suppression}
Consider two wrong trajectories $z_1,z_2$ with confidences $p(z_1)>p(z_2)$. Under the main-text reward, $r(z_1;x) < r(z_2;x)$, so after one improvement step,
\begin{equation}
\label{eq:app-improvement-result}
\frac{\pi_{\theta'}(z_1 \mid x)}{\pi_{\theta'}(z_2 \mid x)}
<
\frac{\pi_\theta(z_1 \mid x)}{\pi_\theta(z_2 \mid x)}.
\end{equation}
That is, among incorrect trajectories, the more overconfident one is suppressed more strongly. Symmetrically, among correct trajectories, higher-confidence ones are relatively amplified.
\end{corollary}

\begin{corollary}[Answer improvement without new knowledge]
\label{cor:app-answer-improvement}
Let
\begin{equation}
\label{eq:app-confidence-weighted-score}
S_\theta(y \mid x)
=
\sum_{z:g(z)=y} \pi_\theta(z \mid x)\, p(z),
\end{equation}
denote the confidence-weighted score of answer $y$, and define the margin
\begin{equation}
\label{eq:app-answer-margin}
\Gamma_\theta(x)
=
S_\theta(y^\star \mid x)
-
\max_{y \neq y^\star} S_\theta(y \mid x),
\end{equation}
where $y^\star$ is the correct answer. If the update increases the relative mass of correct high-confidence trajectories enough that $\Gamma_{\theta'}(x) > 0$ while $\Gamma_\theta(x) \le 0$, then the model's prediction flips from incorrect to correct without introducing any new reasoning trajectory.
\end{corollary}

\begin{proof}[Proof of Proposition~\ref{prop:one-step-relative-improvement}]
For any two trajectories $z_1,z_2$ for the same input $x$, using the normalized form above,
\begin{equation}
\label{eq:app-proof-prop1-ratio}
\frac{\pi_{\theta'}(z_1 \mid x)}{\pi_{\theta'}(z_2 \mid x)}
=
\frac{
\pi_\theta(z_1 \mid x)\exp\!\bigl(\eta\, r(z_1;x)\bigr) / Z(x)
}{
\pi_\theta(z_2 \mid x)\exp\!\bigl(\eta\, r(z_2;x)\bigr) / Z(x)
}.
\end{equation}
The normalization constant cancels, giving
\begin{equation}
\label{eq:app-proof-prop1-cancel}
\frac{\pi_{\theta'}(z_1 \mid x)}{\pi_{\theta'}(z_2 \mid x)}
=
\frac{\pi_\theta(z_1 \mid x)}{\pi_\theta(z_2 \mid x)}
\exp\!\Bigl(\eta\bigl(r(z_1;x)-r(z_2;x)\bigr)\Bigr).
\end{equation}
Taking logarithms yields
\begin{equation}
\label{eq:app-proof-prop1-logodds}
\log \frac{\pi_{\theta'}(z_1 \mid x)}{\pi_{\theta'}(z_2 \mid x)}
=
\log \frac{\pi_\theta(z_1 \mid x)}{\pi_\theta(z_2 \mid x)}
+
\eta \bigl(r(z_1;x)-r(z_2;x)\bigr).
\end{equation}
Therefore, whenever $r(z_1;x)>r(z_2;x)$, the post-update log-odds of $z_1$ against $z_2$ increase; when $r(z_1;x)<r(z_2;x)$, they decrease.
\end{proof}

\begin{proof}[Proof of Corollary~\ref{cor:app-selective-suppression}]
Consider two wrong trajectories $z_1,z_2$ with confidences $p(z_1)>p(z_2)$. For wrong trajectories, the main-text reward is
\begin{equation}
\label{eq:app-proof-cor1-wrong-reward}
r(z;x) = -p(z)^2,
\end{equation}
which is strictly decreasing in $p(z)$ on $[0,1]$. Thus, if $p(z_1) > p(z_2)$, then
\begin{equation}
\label{eq:app-proof-cor1-reward-order}
r(z_1;x) = -p(z_1)^2 < -p(z_2)^2 = r(z_2;x),
\end{equation}
so the higher-confidence wrong trajectory is downweighted more strongly. Applying Proposition~\ref{prop:one-step-relative-improvement},
\begin{equation}
\label{eq:app-proof-cor1-logodds}
\log \frac{\pi_{\theta'}(z_1 \mid x)}{\pi_{\theta'}(z_2 \mid x)}
=
\log \frac{\pi_\theta(z_1 \mid x)}{\pi_\theta(z_2 \mid x)}
+
\eta\bigl(r(z_1;x)-r(z_2;x)\bigr),
\end{equation}
and the final term is strictly negative. Therefore,
\begin{equation}
\label{eq:app-proof-cor1-logodds-drop}
\log \frac{\pi_{\theta'}(z_1 \mid x)}{\pi_{\theta'}(z_2 \mid x)}
<
\log \frac{\pi_\theta(z_1 \mid x)}{\pi_\theta(z_2 \mid x)},
\end{equation}
which implies
\begin{equation}
\label{eq:app-proof-cor1-ratio-drop}
\frac{\pi_{\theta'}(z_1 \mid x)}{\pi_{\theta'}(z_2 \mid x)}
<
\frac{\pi_\theta(z_1 \mid x)}{\pi_\theta(z_2 \mid x)}.
\end{equation}
Thus, among incorrect trajectories, the more overconfident one is suppressed more strongly.

For correct trajectories, the reward is
\begin{equation}
\label{eq:app-proof-cor1-correct-order}
r(z;x) = 2p(z) - p(z)^2,
\end{equation}
which is increasing in $p(z)$ on $[0,1]$. Hence if $z_1,z_2$ are both correct with $p(z_1) > p(z_2)$, then $r(z_1;x) > r(z_2;x)$, and Proposition~\ref{prop:one-step-relative-improvement} implies that the relative likelihood of $z_1$ increases after the update. Higher-confidence correct trajectories are therefore relatively amplified.
\end{proof}

\begin{proof}[Proof of Corollary~\ref{cor:app-answer-improvement}]
Recall the confidence-weighted answer score
\begin{equation}
\label{eq:app-proof-cor2-score}
S_\theta(y \mid x)
=
\sum_{z:g(z)=y} \pi_\theta(z \mid x)\, p(z),
\end{equation}
and the answer margin
\begin{equation}
\label{eq:app-proof-cor2-margin}
\Gamma_\theta(x)
=
S_\theta(y^\star \mid x)
-
\max_{y \neq y^\star} S_\theta(y \mid x),
\end{equation}
where $y^\star$ is the correct answer.

Suppose that before the GRPO update,
\begin{equation}
\label{eq:app-proof-cor2-margin-before}
\Gamma_\theta(x)\le 0,
\end{equation}
so the correct answer does not strictly dominate all competing answers under the confidence-weighted score. Suppose further that after the update,
\begin{equation}
\label{eq:app-proof-cor2-margin-after}
\Gamma_{\theta'}(x)>0.
\end{equation}
Then by definition,
\begin{equation}
\label{eq:app-proof-cor2-score-domination}
\begin{aligned}
S_{\theta'}(y^\star \mid x)
>
S_{\theta'}(y \mid x)
\qquad
\text{for all } y \neq y^\star.
\end{aligned}
\end{equation}
Therefore the confidence-weighted decision rule
\begin{equation}
\label{eq:app-proof-cor2-decision-rule}
\hat y_\theta(x)=\arg\max_y S_\theta(y \mid x)
\end{equation}
changes from not selecting $y^\star$ before the update to selecting $y^\star$ after the update.

Finally, under the tilted-distribution model, the update only changes the relative weights of existing trajectories through $\pi_{\theta'}(z\mid x)$; it does not introduce new trajectory support. Hence the prediction flips from incorrect to correct without requiring any new reasoning trajectory to be created.
\end{proof}

\paragraph{Remark.}
These proofs establish only the consequences of the one-step tilted-distribution model. They should therefore be interpreted as a local analytical account of how uncertainty-aware RL can improve calibration and answer selection by redistributing probability mass over existing trajectories, rather than as an exact global characterization of GRPO training dynamics. 

\subsection{Proof of the support-preserving answer reweighting theorem}

We restate the theorem for convenience.

\begin{theorem}[Latent-answer extraction under support-preserving reweighting]
\label{thm:latent_answer_extraction_appendix}
Fix an input $x$, and let $\pi_\theta(z\mid x)$ be the model distribution over complete reasoning trajectories $z$. Suppose the post-update policy is given by
\begin{equation}
\pi_{\theta'}(z\mid x)
=
\frac{\pi_\theta(z\mid x)\exp\!\bigl(\eta r(z;x)\bigr)}
{Z(x)},
\label{eq:tilted_policy_appendix}
\end{equation}
where
\begin{equation}
Z(x)
=
\sum_{z}\pi_\theta(z\mid x)\exp\!\bigl(\eta r(z;x)\bigr),
\label{eq:partition_appendix}
\end{equation}
and $\eta>0$. Define the answer-level probability mass
\begin{equation}
M_\theta(y\mid x)
=
\sum_{z:\,g(z)=y}\pi_\theta(z\mid x),
\label{eq:answer_mass_appendix}
\end{equation}
where $g(z)$ denotes the final answer induced by trajectory $z$.
Let $y^\star$ be the correct answer and let $\bar y\neq y^\star$ be any competing wrong answer. Assume that every trajectory producing the correct answer satisfies
\begin{equation}
r(z;x)\ge a,
\qquad \forall z \text{ such that } g(z)=y^\star,
\label{eq:correct_reward_lb_appendix}
\end{equation}
and every trajectory producing $\bar y$ satisfies
\begin{equation}
r(z;x)\le b,
\qquad \forall z \text{ such that } g(z)=\bar y,
\label{eq:wrong_reward_ub_appendix}
\end{equation}
for constants $a>b$. Then:
\begin{equation}
\frac{M_{\theta'}(y^\star\mid x)}{M_{\theta'}(\bar y\mid x)}
\ge
\exp\!\bigl(\eta(a-b)\bigr)
\frac{M_\theta(y^\star\mid x)}{M_\theta(\bar y\mid x)}.
\label{eq:mass_ratio_bound_appendix}
\end{equation}
Moreover,
\begin{equation}
\operatorname{supp}\pi_{\theta'}(\cdot\mid x)
=
\operatorname{supp}\pi_\theta(\cdot\mid x).
\label{eq:support_preserve_appendix}
\end{equation}
\end{theorem}

\begin{proof}
We first prove the answer-mass ratio bound. By definition,
\begin{align}
M_{\theta'}(y^\star\mid x)
&=
\sum_{z:\,g(z)=y^\star}\pi_{\theta'}(z\mid x) \\
&=
\sum_{z:\,g(z)=y^\star}
\frac{\pi_\theta(z\mid x)\exp\!\bigl(\eta r(z;x)\bigr)}{Z(x)}.
\label{eq:correct_mass_start}
\end{align}
For every trajectory $z$ such that $g(z)=y^\star$, the assumption
in Eq.~\eqref{eq:correct_reward_lb_appendix} implies
\begin{equation}
\exp\!\bigl(\eta r(z;x)\bigr)\ge \exp(\eta a).
\label{eq:correct_exp_lb}
\end{equation}
Substituting this into Eq.~\eqref{eq:correct_mass_start} gives
\begin{align}
M_{\theta'}(y^\star\mid x)
&\ge
\sum_{z:\,g(z)=y^\star}
\frac{\pi_\theta(z\mid x)\exp(\eta a)}{Z(x)} \\
&=
\frac{\exp(\eta a)}{Z(x)}
\sum_{z:\,g(z)=y^\star}\pi_\theta(z\mid x) \\
&=
\frac{\exp(\eta a)}{Z(x)}\,M_\theta(y^\star\mid x).
\label{eq:correct_mass_lower}
\end{align}

Similarly,
\begin{align}
M_{\theta'}(\bar y\mid x)
&=
\sum_{z:\,g(z)=\bar y}\pi_{\theta'}(z\mid x) \\
&=
\sum_{z:\,g(z)=\bar y}
\frac{\pi_\theta(z\mid x)\exp\!\bigl(\eta r(z;x)\bigr)}{Z(x)}.
\label{eq:wrong_mass_start}
\end{align}
For every trajectory $z$ such that $g(z)=\bar y$, the assumption
in Eq.~\eqref{eq:wrong_reward_ub_appendix} implies
\begin{equation}
\exp\!\bigl(\eta r(z;x)\bigr)\le \exp(\eta b).
\label{eq:wrong_exp_ub}
\end{equation}
Therefore,
\begin{align}
M_{\theta'}(\bar y\mid x)
&\le
\sum_{z:\,g(z)=\bar y}
\frac{\pi_\theta(z\mid x)\exp(\eta b)}{Z(x)} \\
&=
\frac{\exp(\eta b)}{Z(x)}
\sum_{z:\,g(z)=\bar y}\pi_\theta(z\mid x) \\
&=
\frac{\exp(\eta b)}{Z(x)}\,M_\theta(\bar y\mid x).
\label{eq:wrong_mass_upper}
\end{align}

Combining Eq.~\eqref{eq:correct_mass_lower} and Eq.~\eqref{eq:wrong_mass_upper}, we obtain
\begin{align}
\frac{M_{\theta'}(y^\star\mid x)}{M_{\theta'}(\bar y\mid x)}
&\ge
\frac{\frac{\exp(\eta a)}{Z(x)}M_\theta(y^\star\mid x)}
{\frac{\exp(\eta b)}{Z(x)}M_\theta(\bar y\mid x)} \\
&=
\exp\!\bigl(\eta(a-b)\bigr)
\frac{M_\theta(y^\star\mid x)}{M_\theta(\bar y\mid x)}.
\end{align}
This proves Eq.~\eqref{eq:mass_ratio_bound_appendix}.

We now prove support preservation. Since $\exp(\eta r(z;x))>0$ for every trajectory $z$ and
$Z(x)>0$, Eq.~\eqref{eq:tilted_policy_appendix} implies
\begin{equation}
\pi_{\theta'}(z\mid x)>0
\quad \Longleftrightarrow \quad
\pi_\theta(z\mid x)>0.
\label{eq:support_equiv_step}
\end{equation}
Therefore,
\begin{equation}
\operatorname{supp}\pi_{\theta'}(\cdot\mid x)
=
\operatorname{supp}\pi_\theta(\cdot\mid x),
\end{equation}
which proves Eq.~\eqref{eq:support_preserve_appendix}.

Finally, Eq.~\eqref{eq:support_preserve_appendix} yields the main interpretation used in the paper:
the update cannot create a correct trajectory that was absent from the original support. It can only
increase the relative mass of correct trajectories that were already present but underweighted under
$\pi_\theta$.
\end{proof}

\paragraph{Specialization to the verbal-confidence reward.}
Under the Brier-style reward
\begin{equation}
r(z;x)=
\begin{cases}
2p(z) - p(z)^2, & g(z)=y^\star,\\
-p(z)^2, & g(z)\neq y^\star,
\end{cases}
\label{eq:verbal_reward_appendix}
\end{equation}
suppose every correct-answer trajectory satisfies $p(z) \ge \alpha$ and every trajectory producing
$\bar y$ satisfies $p(z) \le \beta$. Since $2p - p^2$ is increasing in $p$ on $[0,1]$ and $-p^2$ is strictly decreasing in $p$ on $[0,1]$, we may choose
\begin{equation}
a = 2\alpha - \alpha^2,
\qquad
b = 0.
\label{eq:ab_specialization_appendix}
\end{equation}
Substituting into Eq.~\eqref{eq:mass_ratio_bound_appendix} yields
\begin{equation}
\frac{M_{\theta'}(y^\star\mid x)}{M_{\theta'}(\bar y\mid x)}
\ge
\exp\!\bigl(\eta(2\alpha - \alpha^2)\bigr)
\frac{M_\theta(y^\star\mid x)}{M_\theta(\bar y\mid x)}.
\label{eq:verbal_specialized_bound_appendix}
\end{equation}
Note that the bound depends on $\alpha$ but not on $\beta$. The assumption $p(z) \le \beta$ only restricts how high wrong-trajectory confidence can go, not how low it can go; since $r = -p^2$ is largest when $p$ is small, the upper bound $b$ over wrong trajectories is $0$ regardless of $\beta$. Thus, when correct trajectories are highly confident (large $\alpha$), wrong trajectories are exponentially downweighted relative to correct trajectories, which formalizes the intuition that the objective acts as an anti-overconfidence filter.

\section{Additional Mechanistic Evidence}
\label{app:mechanistic}

This appendix provides additional mechanistic detail supporting the main-text claim that calibration is integrated into the reasoning process rather than appended as a purely superficial output-formatting step. The appendix has two goals. First, it expands several analyses that are informative but not central enough for the main body. Second, it clarifies the limits of what the current experiments do and do not show. In particular, these analyses support a mechanistic account of uncertainty-aware reasoning, but they do not by themselves prove that the emitted confidence is the true posterior probability that the answer is correct.

\subsection{Expanded Token-Level Divergence Analysis}
\label{app:token-divergence}

The main text emphasizes per-token localization, since this is the clearest way to show that calibration-induced change is concentrated at uncertainty-related positions. A complementary view is to examine the \emph{total} KL mass allocated to each token type. This diagnostic is useful, but it must be interpreted carefully because long reasoning spans dominate total mass simply by occupying many more positions than uncertainty tokens.

\begin{figure*}[htbp]
    \centering
    \begin{subfigure}[t]{0.48\textwidth}
        \centering
        \includegraphics[width=\linewidth]{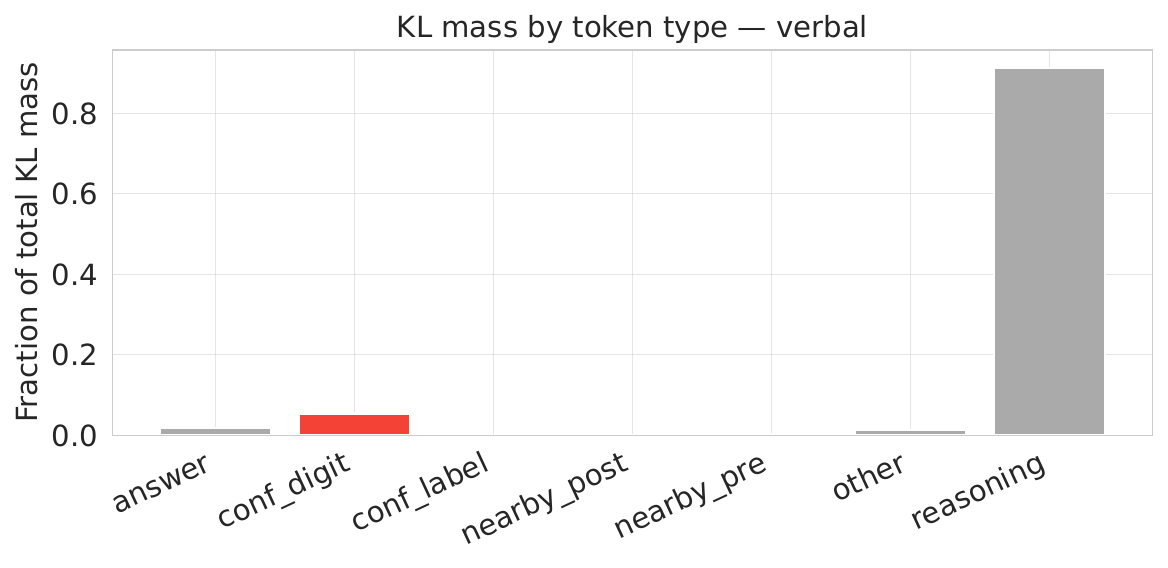}
        \caption{Verbalized confidence. Most total KL mass lies in reasoning tokens, even though confidence digits are strongly enriched on a per-token basis.}
        \label{fig:app-exp1-mass-verbal}
    \end{subfigure}
    \hfill
    \begin{subfigure}[t]{0.48\textwidth}
        \centering
        \includegraphics[width=\linewidth]{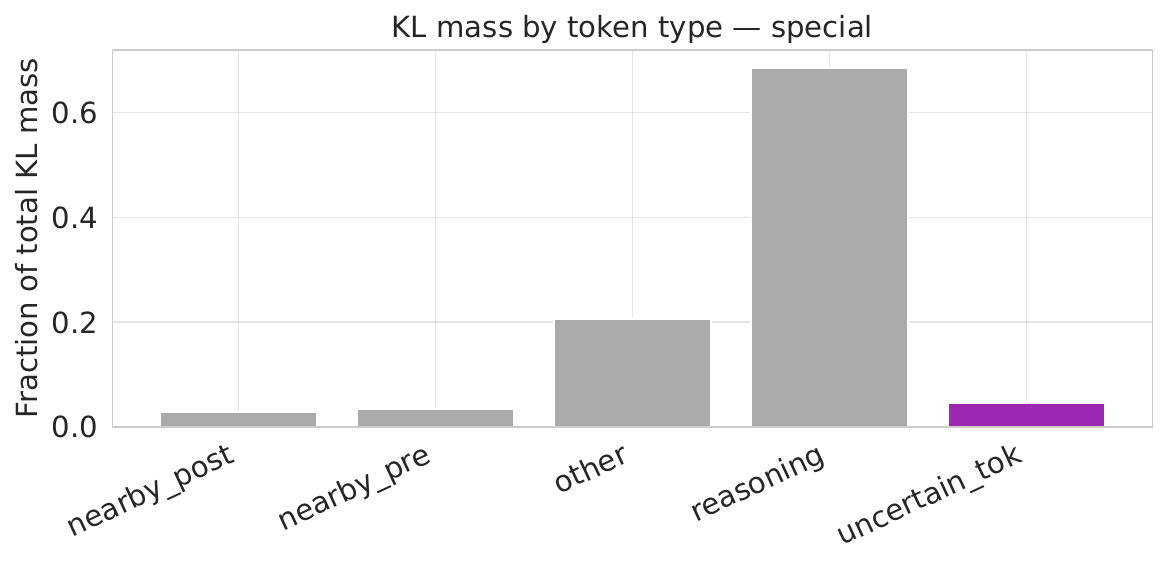}
        \caption{\texttt{<uncertain>} marker. The \texttt{<uncertain>} token is enriched, but total KL mass remains dominated by reasoning and nearby context tokens.}
        \label{fig:app-exp1-mass-special}
    \end{subfigure}
    \caption{\textbf{KL mass fractions by token type.} These plots complement the main-text boxplots by showing total KL allocation rather than per-token enrichment. Because reasoning spans are much longer than uncertainty spans, raw mass fractions should be interpreted together with the per-token statistics in the main text.}
    \label{fig:app-exp1-mass}
\end{figure*}

Figure~\ref{fig:app-exp1-mass} shows why the per-token view is the right primary lens. Under verbalized confidence, the confidence digit is strongly enriched relative to ordinary reasoning tokens, but reasoning still accounts for most KL mass because it occupies far more positions. The same logic holds under the \texttt{<uncertain>} marker: the \texttt{<uncertain>} token is a meaningful concentration point, but the surrounding reasoning sequence still carries most of the aggregate divergence. This is consistent with a mechanism in which uncertainty is computed across the reasoning trace and only becomes especially visible at a small number of output positions.

Two additional details matter for interpretation. First, under verbalized confidence, the \texttt{Confidence:} label itself is essentially inert, reinforcing the conclusion that training altered the emitted scalar value rather than the output template. Second, under the \texttt{<uncertain>} marker, the model shows elevated KL in the nearby pre- and post-windows around the emission position, which is not seen under verbalized confidence. This broader footprint suggests that under the \texttt{<uncertain>} marker, the model enters a local uncertainty-related computation regime around the emission event, whereas under verbalized confidence the model behaves more like a clean endpoint readout.

A caveat is that these token-type summaries are affected by sequence truncation. The current analyses used \texttt{max\_seq\_len=512}, and the main analysis already noted that this likely increases the residual \texttt{other} category and may undercount some structured output regions. This does not undermine the core localization result, but it does mean that the exact token-type mass fractions should be treated as approximate.

\subsection{Hidden-State Patching as Supporting Evidence}
\label{app:patching}

A natural but incorrect inference from token-level localization is that the
uncertainty token itself contains the full causal mechanism. The activation-
patching results provide supporting evidence against that stronger claim.
Localization identifies where the calibration effect becomes most visible in
the output distribution, but not necessarily where that effect is fully
computed.

\begin{figure*}[htbp]
    \centering
    \begin{subfigure}[t]{0.48\textwidth}
        \centering
        \includegraphics[width=\linewidth]{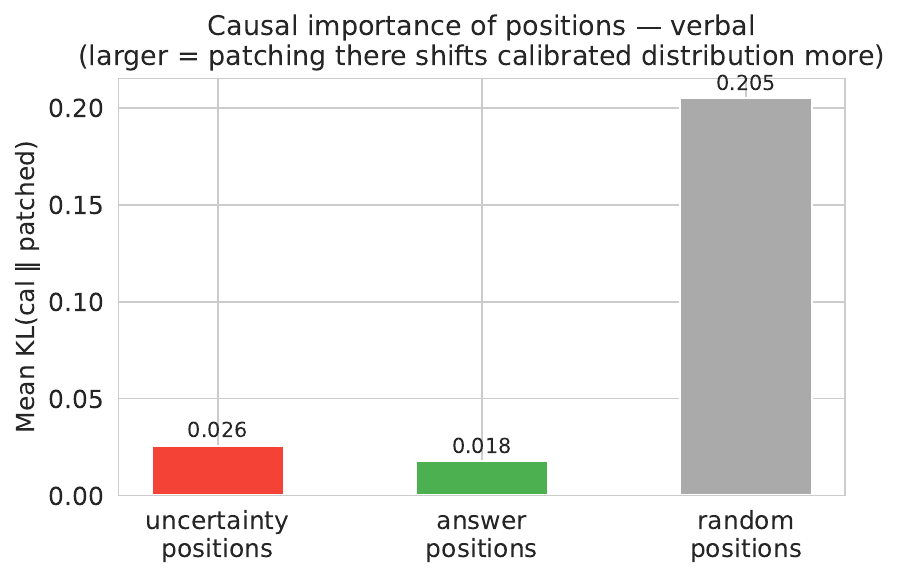}
        \caption{Verbalized confidence}
        \label{fig:app-exp2-verbal}
    \end{subfigure}
    \hfill
    \begin{subfigure}[t]{0.48\textwidth}
        \centering
        \includegraphics[width=\linewidth]{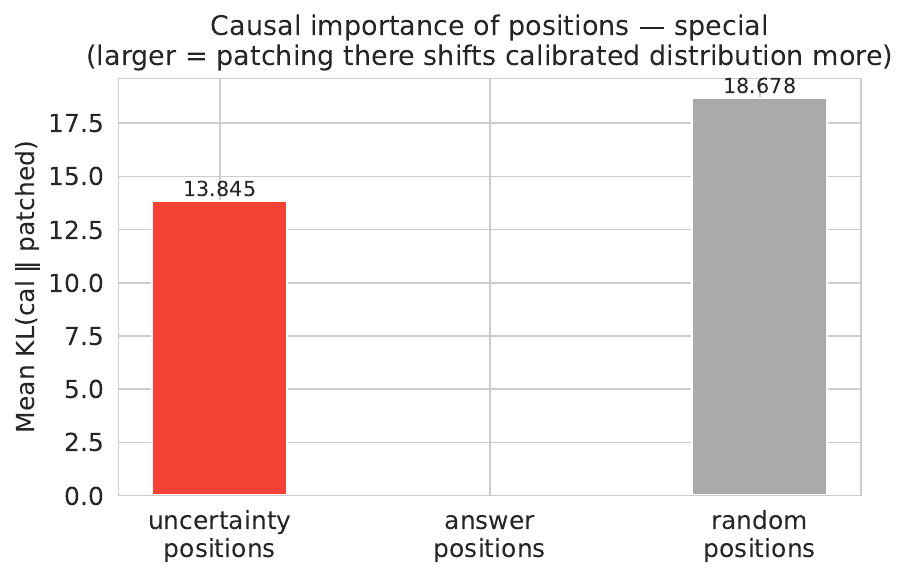}
        \caption{\texttt{<uncertain>} marker}
        \label{fig:app-exp2-special}
    \end{subfigure}
    \caption{\textbf{Hidden-state patching at the signal position versus reasoning positions.} In both methods, patching the reasoning trace is at least as disruptive as patching the signal token itself, which is consistent with uncertainty being assembled across the trajectory and only exposed at the designated output position.}
    \label{fig:app-patching}
\end{figure*}

For verbalized confidence, patching hidden states at confidence-digit positions
produces almost no disruption. In contrast, patching random reasoning
positions produces substantially larger changes on average. This pattern is
consistent with the confidence value being read out from information that has
already been assembled across earlier reasoning tokens and stored in the
accumulated attention state.

For the \texttt{<uncertain>} marker, patching the \texttt{<uncertain>} position does matter,
but it is still not the dominant causal locus under the current intervention.
Random reasoning positions are even more disruptive on average. This indicates
that the explicit uncertainty marker participates in the mechanism, but does
not by itself define the full uncertainty computation. The marker is part
of a broader process rather than a self-contained switch.

This distinction helps clarify the phrase \emph{localized but distributed}.
The calibration effect is localized in the sense that it becomes especially
visible at uncertainty-related output positions. However, the supporting
computation is distributed over the reasoning trajectory that precedes those
positions. Because the current intervention changes only one marker-position state,
we view this analysis as suggestive supporting evidence rather than a complete
causal account.

\subsection{Parameter-Space Drift and Embedding Repositioning}
\label{app:drift}

The weight-drift analysis addresses an important question left open by the representation results: if verbalized confidence preserves the base geometry so strongly, did the model simply undergo a much smaller parameter update than under the \texttt{<uncertain>} marker? The answer is no. Both models exhibit broadly similar parameter-space drift patterns, which makes their difference in representation-space behavior more striking.

\begin{figure*}[t]
    \centering
    \begin{subfigure}[t]{0.48\textwidth}
        \centering
        \includegraphics[width=\linewidth]{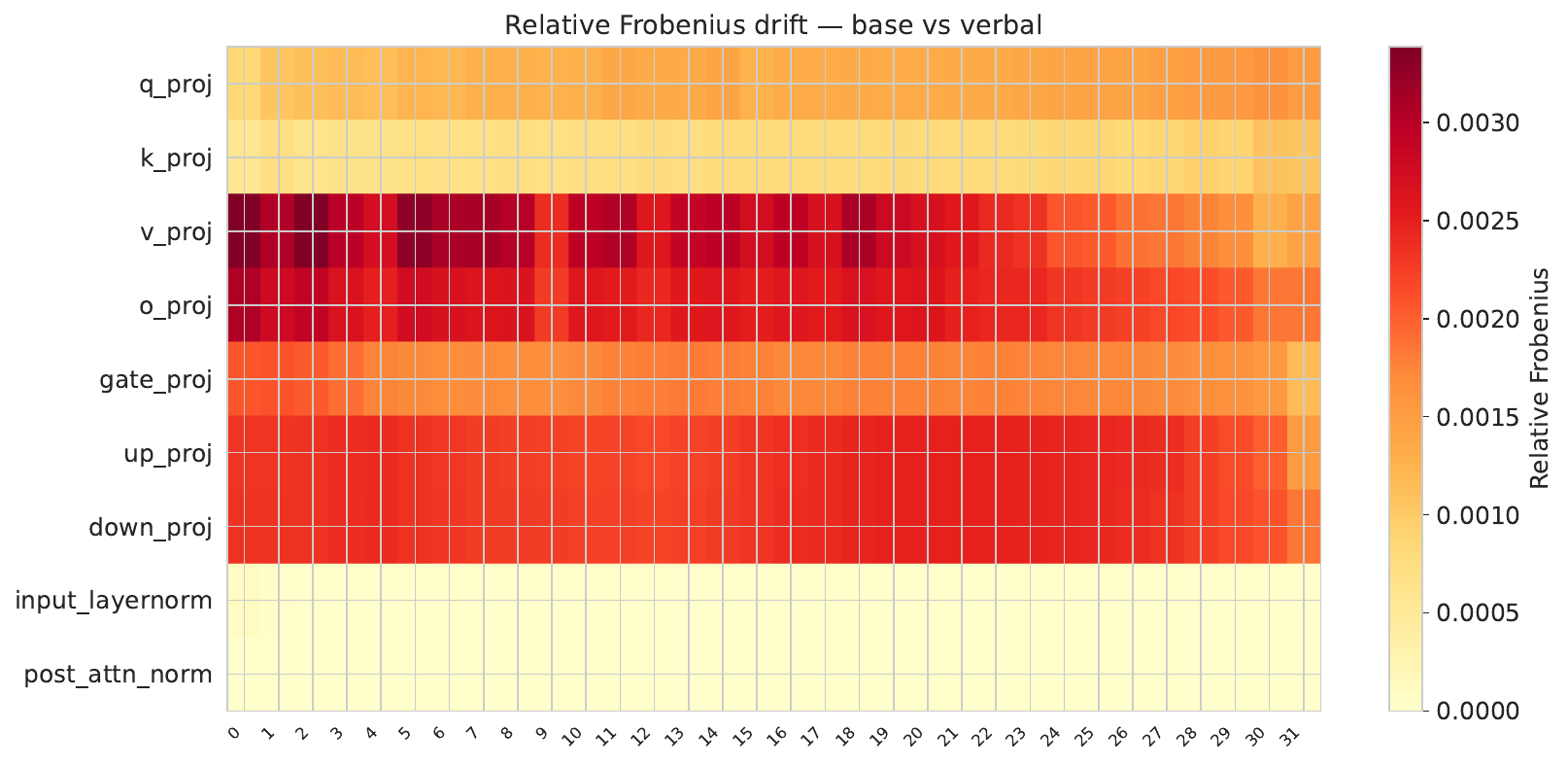}
        \caption{Verbalized confidence. Drift is concentrated in value/output projections and MLP projections, with minimal change in normalization layers.}
        \label{fig:app-exp4-verbal}
    \end{subfigure}
    \hfill
    \begin{subfigure}[t]{0.48\textwidth}
        \centering
        \includegraphics[width=\linewidth]{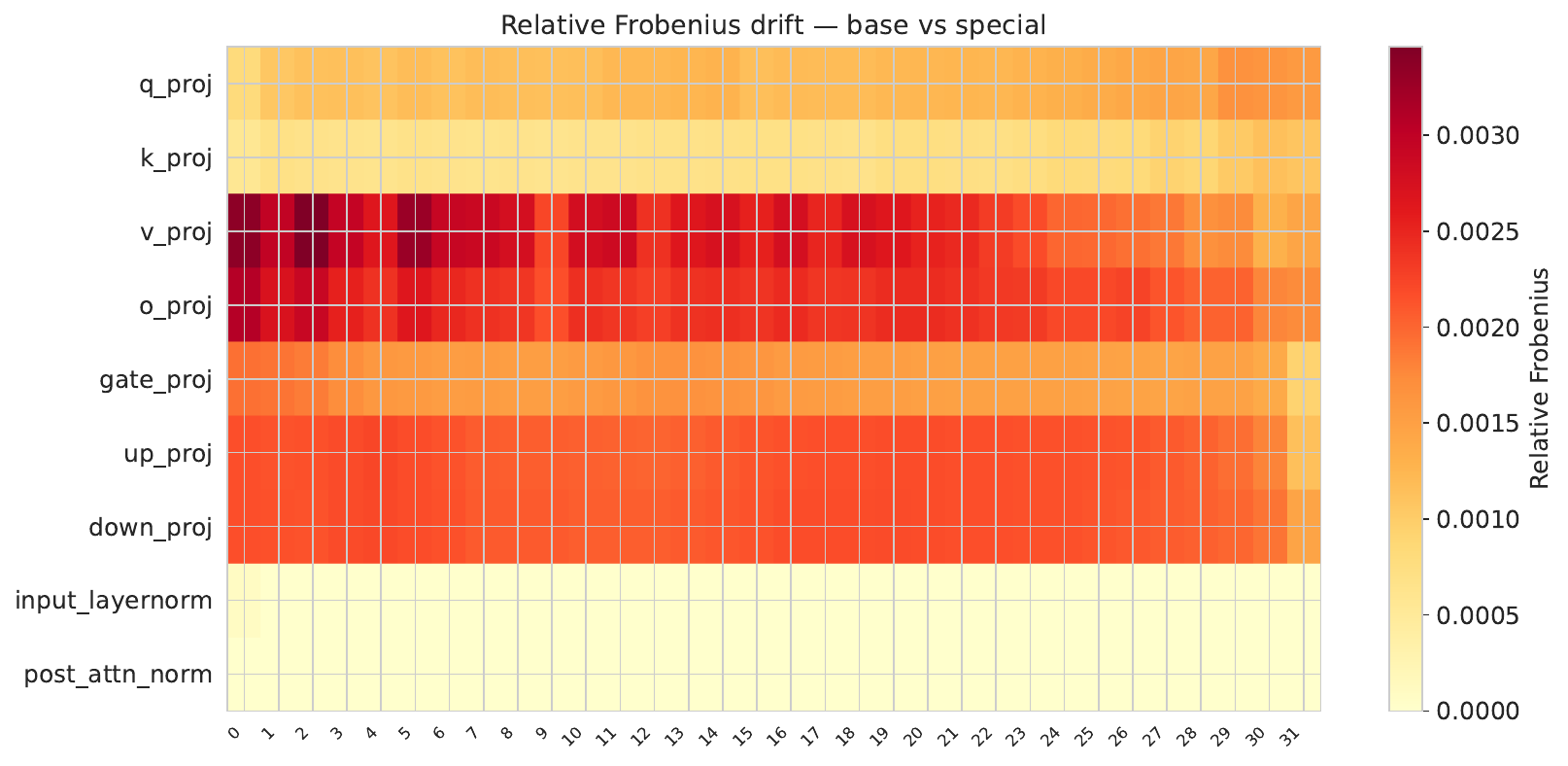}
        \caption{\texttt{<uncertain>} marker. A similar module-level drift pattern appears despite stronger late-layer representational divergence.}
        \label{fig:app-exp4-special}
    \end{subfigure}
    \caption{\textbf{Relative Frobenius weight drift across layers and module types.} Both calibrated models show similar update structure in parameter space, with the largest changes in \texttt{v\_proj}, \texttt{o\_proj}, and MLP projection layers. This makes the difference in representation geometry especially noteworthy: similar magnitudes of weight drift yield very different geometric consequences.}
    \label{fig:app-exp4-drift}
\end{figure*}

Figure~\ref{fig:app-exp4-drift} shows that both calibrated models place most of their parameter drift in the same broad module classes, especially the attention value/output projections and MLP projections. LayerNorm terms change very little. This pattern is similar across the two methods, and the overall update magnitudes are also comparable. Thus, the fact that the CKA under verbalized confidence remains essentially unchanged while the CKA under the \texttt{<uncertain>} marker diverges in late layers cannot be explained simply by one model being updated much more than the other.

This produces an informative contrast. Under verbalized confidence, similarly sized weight-space changes largely preserve local representation geometry at the uncertainty readout position. Under the \texttt{<uncertain>} marker, similarly sized changes accumulate into more visible late-layer geometric divergence. One interpretation is that the inductive bias of the training objective matters as much as raw update size: a trajectory-level scalar-confidence objective can be realized through a relatively geometry-preserving readout adjustment, whereas an explicit mid-reasoning uncertainty marker encourages a deeper reorganization of the computation that produces that marker.

The embedding-drift analysis reinforces this distinction. Under the \texttt{<uncertain>} marker, the token embeddings corresponding to the components of \texttt{<uncertain>} drift more than a random-token baseline, consistent with targeted repositioning of the explicit uncertainty marker. Under verbalized confidence, those same component tokens drift less than baseline on average, and common bracket tokens are unchanged. This further supports the interpretation that verbalized confidence does not rely on explicit uncertainty-token specialization, whereas the \texttt{<uncertain>} marker does.

\subsection{Mechanism-to-Utility Linkage and Its Limits}
\label{app:linkage}

The internal mechanism analysis in the main text establishes where uncertainty-related
changes occur and how deeply they alter the model's internal states. A
remaining question is whether those measurements are also predictive at the
level of individual examples, rather than only in aggregate.

\begin{figure*}[t]
    \centering
    \includegraphics[width=\textwidth]{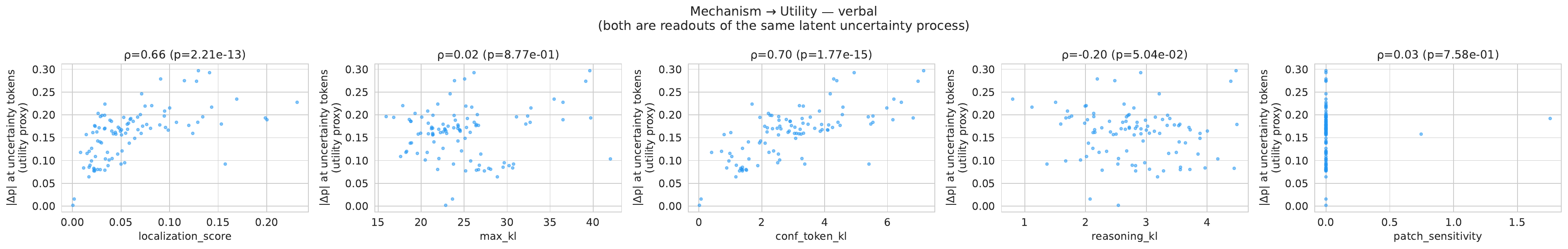}
    \caption{\textbf{Mechanism-to-behavior linkage for verbalized confidence.} Localization-related features predict per-example confidence shifts with cross-validated \(R^2 = 0.51\), indicating that the strength of the learned confidence mechanism varies meaningfully across examples rather than appearing only as a population-level average.}
    \label{fig:main-exp5-verbal}
\end{figure*}

\begin{figure*}[t]
    \centering
    \includegraphics[width=\textwidth]{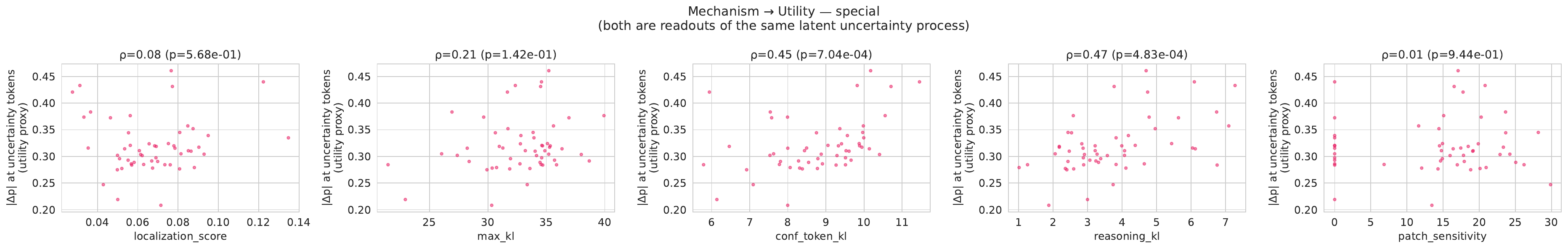}
    \caption{\textbf{Mechanism-to-utility linkage for the \texttt{<uncertain>} marker.} Under the current proxy utility target, the \texttt{<uncertain>}-marker model shows weak and unstable within-subset linkage. This likely reflects a mismatch between the current utility proxy and the model's operative mechanism, which is more naturally framed as a binary emission decision than as graded variation within already-emitting examples.}
    \label{fig:app-exp5-special}
\end{figure*}

For verbalized confidence, Figure~\ref{fig:main-exp5-verbal} shows that the
localization structure captured in the distributional analysis is not an
artifact of averaging: it varies meaningfully across examples and predicts how
strongly the confidence output differs from the base model on each individual
instance. Examples in the top localization quartile show confidence shifts
$86\%$ larger than those in the bottom quartile ($0.207$ vs.\ $0.111$). This
supports the interpretation that verbalized-confidence training is not merely learning a
surface format; the model is genuinely learning when to engage a stronger
confidence adjustment based on the information accumulated in the reasoning
trace.

However, the current utility target is still a proxy: it is the magnitude of change at the uncertainty token, not the true probability that the answer is correct. Therefore, the verbalized-confidence linkage result should be interpreted as evidence that the model learns a structured scalar uncertainty readout from the reasoning trajectory, rather than as proof that the emitted scalar is already a perfectly calibrated posterior probability of correctness.

Figure~\ref{fig:app-exp5-special} clarifies why the same analysis is weak for the \texttt{<uncertain>} marker. The \texttt{<uncertain>} marker is likely governed by a different operative mechanism. The important decision is whether to emit \texttt{<uncertain>} at all, rather than how much to vary a continuous confidence value \emph{within} the subset of examples that already emitted the marker. Under that interpretation, a within-emission regression is simply not the best target. A stronger analysis for that method would instead compare emitting and non-emitting examples directly, treating uncertainty emission as a selective-prediction or abstention-like decision.

\subsection{Summary}
\label{app:summary}

The additional analyses in this appendix reinforce three points. First, the localization observed in the main text is real, but should be understood in per-token rather than raw-mass terms. Second, localization does not imply that the uncertainty token itself is the complete causal mechanism; the supporting computation remains distributed across the reasoning trajectory. Third, verbalized confidence and the \texttt{<uncertain>} marker differ not in whether they affect uncertainty, but in how deeply they rewrite the computation that supports it. Verbalized confidence is consistent with a geometry-preserving readout of distributed uncertainty, whereas the \texttt{<uncertain>} marker is consistent with a more explicit uncertainty mode that is assembled during reasoning and expressed through a dedicated marker.

\section{Supplementary Quantitative Results}
\label{app:quantitative}

\subsection{Detailed Results for Verbalized Confidence}
\label{app:verbalized_confidence_details}

The main text keeps only the quantitative results needed to establish the central claim of verbalized confidence: calibration improves final-answer uncertainty while preserving or slightly improving answer quality. This subsection retains the more detailed breakdowns that support that conclusion but are not needed in the main narrative.

\begin{figure*}[htbp]
    \centering
        \includegraphics[width=\linewidth]{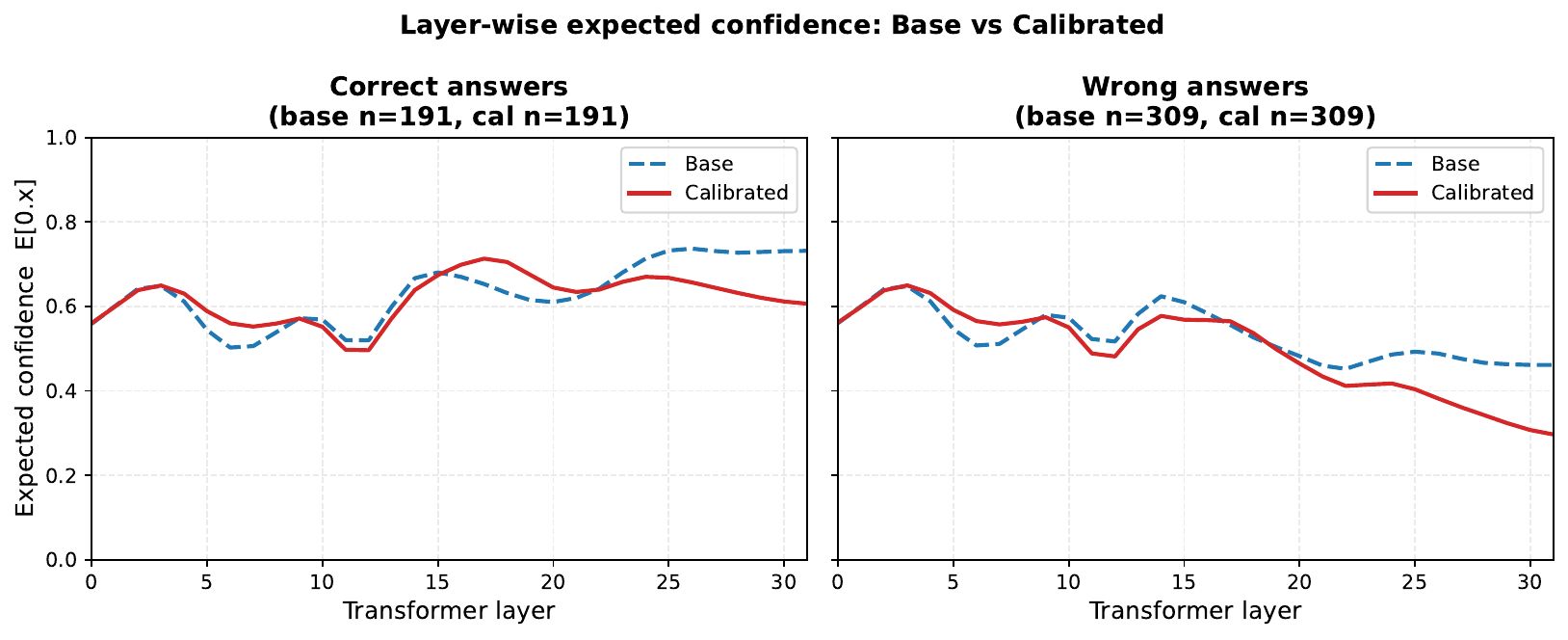}

    \caption{Detailed binned routing view for verbalized confidence calibration. In the base model, correct answers and many wrong answers both terminate with dominant mass in the \textsc{High} confidence bin. After calibration, low-confidence errors are redirected away from \textsc{High} and into \textsc{Low}, while correct answers remain more conservative. This makes the main mechanism visually explicit: calibration sharpens the late-stage mapping from hidden states to confidence outputs rather than uniformly lowering confidence everywhere.}
    \label{fig:app_calibration_routing}
\end{figure*}

\begin{figure}[htbp]
    \centering
    \begin{subfigure}[t]{0.48\linewidth}
        \centering
        \includegraphics[width=\linewidth]{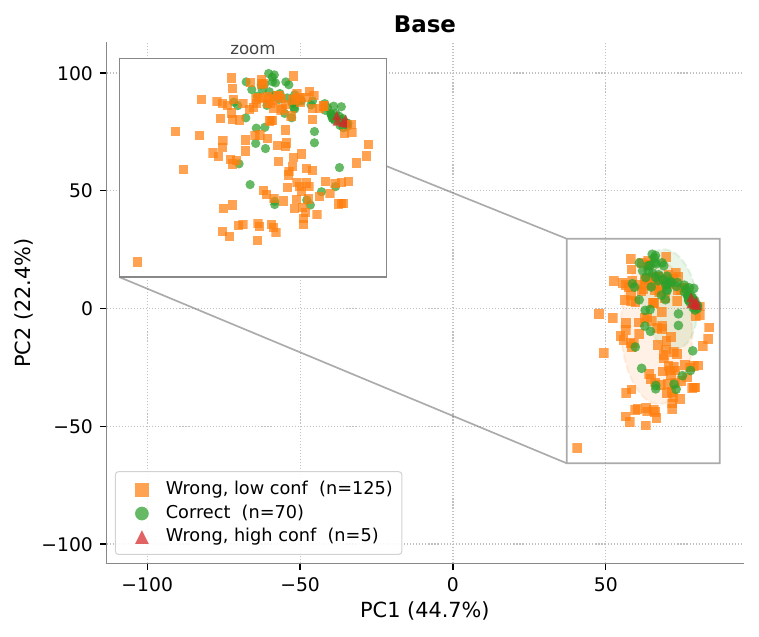}
        \caption{Base model}
        \label{fig:pca_group_base}
    \end{subfigure}
    \hfill
    \begin{subfigure}[t]{0.48\linewidth}
        \centering
        \includegraphics[width=\linewidth]{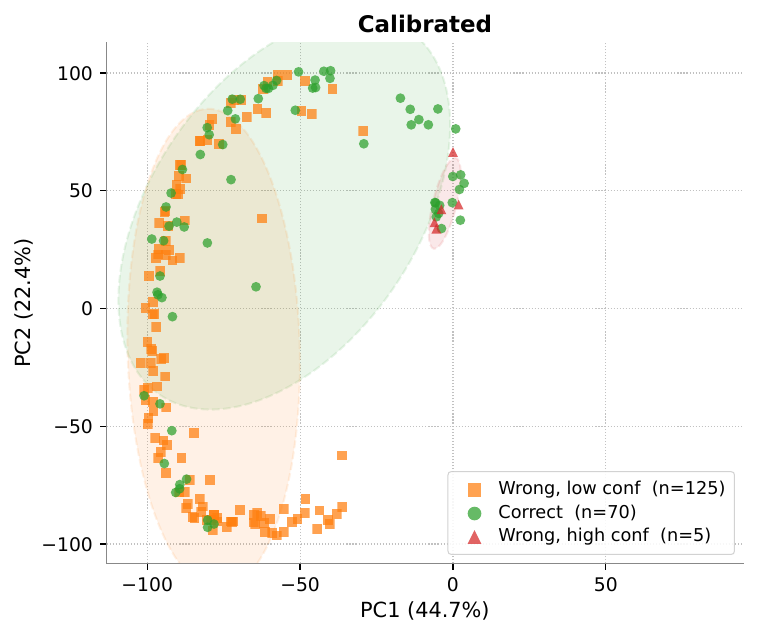}
        \caption{Calibrated model}
        \label{fig:pca_group_cal}
    \end{subfigure}
    \caption{Final-layer PCA of the confidence-token hidden state, grouped by outcome type. In the base model, correct and wrong low-confidence examples remain substantially mixed. After calibration, the representation aligns more cleanly with outcome structure: wrong low-confidence cases concentrate on the low-confidence side of the structure, while correct examples occupy the higher-confidence region more consistently.}
    \label{fig:pca_group_appendix}
\end{figure}
\begin{table}[htbp]
  \centering
  \small
  \caption{Calibration metrics for Llama-3-8B before and after calibration training.}
  \label{tab:calibration}
  \begin{tabular}{lrr}
    \toprule
    Metric & Llama-3-8B (base) & Llama-3-8B (calibrated) \\
    \midrule
    Accuracy ($\uparrow$)                            & 0.345 & 0.358 \\
    Avg.\ verbalized confidence                      & 0.869 & 0.403 \\
    Overconfidence gap (conf $-$ acc) ($\downarrow$) & +0.523 & +0.045 \\
    ECE ($\downarrow$)                               & 0.383 & 0.049 \\
    Brier score ($\downarrow$)                       & 0.504 & 0.166 \\
    NLL (confidence) ($\downarrow$)                  & 4.987 & 0.498 \\
    Parse rate ($\uparrow$)                          & 0.996 & 1.000 \\
    \bottomrule
  \end{tabular}
\end{table}

\begin{figure}[htbp]
    \centering
    \begin{subfigure}[t]{0.48\linewidth}
        \centering
        \includegraphics[width=\linewidth]{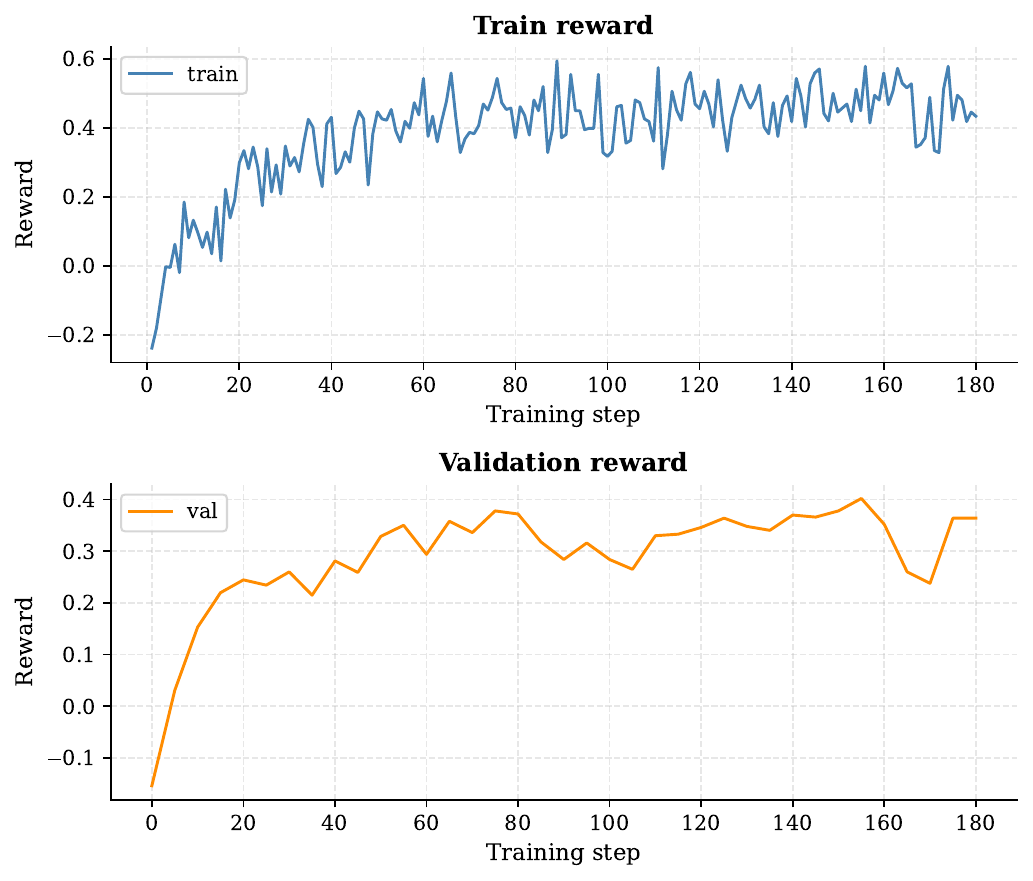}
        \caption{Training and validation reward curves over training steps.}
        \label{fig:training_curve}
    \end{subfigure}
    \hfill
    \begin{subfigure}[t]{0.48\linewidth}
        \centering
        \includegraphics[width=\linewidth]{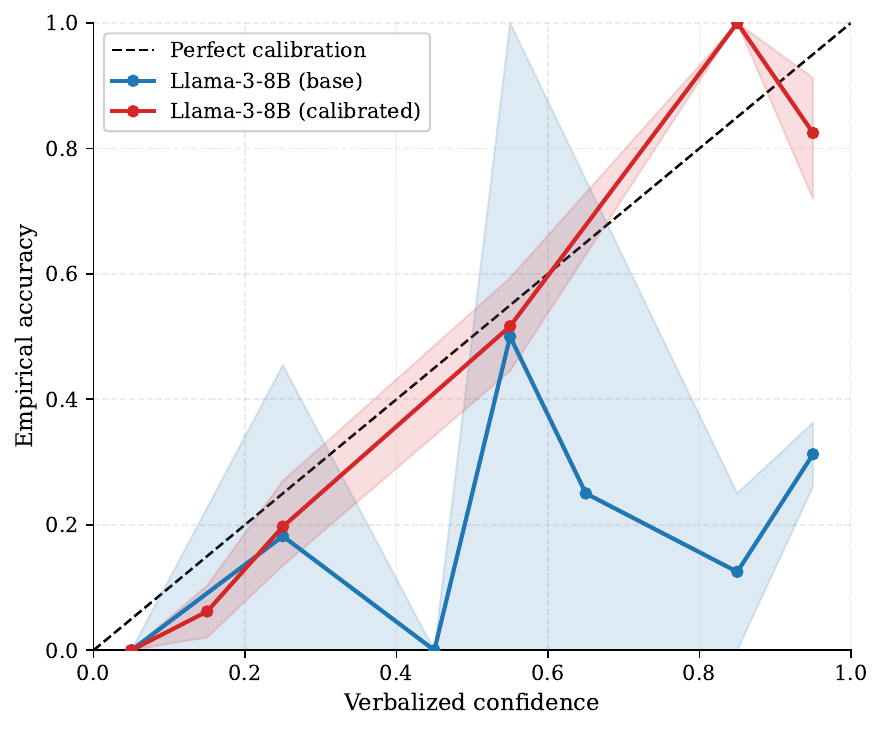}
        \caption{Reliability diagrams for the base and calibrated Llama-3-8B models. Perfect calibration corresponds to the diagonal dashed line.}
        \label{fig:reliability_diagram}
    \end{subfigure}
    \caption{Training dynamics and calibration quality of the calibrated Llama-3-8B model.}
    \label{fig:training_and_reliability}
\end{figure}
\begin{table*}[htbp]
\centering
\scriptsize
\setlength{\tabcolsep}{3.8pt}
\resizebox{\linewidth}{!}{%
\begin{tabular}{l c cc cc cc cc cc}
\toprule
& & \multicolumn{2}{c}{\textbf{Accuracy}} & \multicolumn{2}{c}{\textbf{ECE}} & \multicolumn{2}{c}{\textbf{AUSC}} & \multicolumn{2}{c}{\textbf{Overconf.}} & \multicolumn{2}{c}{\textbf{Conf. on wrong}} \\
\cmidrule(lr){3-4}\cmidrule(lr){5-6}\cmidrule(lr){7-8}\cmidrule(lr){9-10}\cmidrule(lr){11-12}
\textbf{Dataset} & \(n\) & \textbf{Base} & \textbf{Cal.} & \textbf{Base} & \textbf{Cal.} & \textbf{Base} & \textbf{Cal.} & \textbf{Base} & \textbf{Cal.} & \textbf{Base} & \textbf{Cal.} \\
\midrule
2WikiMultihopQA & 500 & 22.4 & \textbf{28.0} & 0.407 & \textbf{0.164} & 0.267 & \textbf{0.458} & 84.8 & \textbf{13.3} & 0.823 & \textbf{0.343} \\
HotpotQA        & 500 & \textbf{35.7} & 35.6 & 0.459 & \textbf{0.085} & 0.477 & \textbf{0.536} & 92.2 & \textbf{4.0}  & 0.854 & \textbf{0.293} \\
MuSiQue         & 500 & 12.1 & \textbf{14.4} & 0.543 & \textbf{0.089} & 0.126 & \textbf{0.224} & 84.6 & \textbf{0.0}  & 0.809 & \textbf{0.207} \\
NQ              & 500 & 46.5 & \textbf{48.2} & 0.367 & \textbf{0.033} & 0.560 & \textbf{0.625} & 96.3 & \textbf{0.4}  & 0.888 & \textbf{0.426} \\
TriviaQA        & 500 & 63.9 & \textbf{68.4} & 0.242 & \textbf{0.211} & 0.749 & \textbf{0.838} & 94.4 & \textbf{0.6}  & 0.866 & \textbf{0.329} \\
\midrule
\textbf{Aggregate} & \textbf{2500} & \textbf{35.4} & \textbf{38.0} & \textbf{0.408} & \textbf{0.119} & \textbf{0.430} & \textbf{0.526} & \textbf{89.9} & \textbf{3.5} & \textbf{0.869} & \textbf{0.360} \\
\bottomrule
\end{tabular}%
}
\caption{Dataset-level summary for verbalized-confidence calibration. ``Overconf.'' denotes the fraction of wrong answers with confidence \(>0.5\).}
\label{tab:vc_dataset_summary}
\end{table*}

\begin{table*}[!htbp]
\centering
\scriptsize

\begin{subtable}[t]{0.35\textwidth}
\centering
\setlength{\tabcolsep}{4pt}
\resizebox{\linewidth}{!}{%
\begin{tabular}{lcc}
\toprule
\textbf{Band} & \textbf{Base} & \textbf{Cal.} \\
\midrule
$ c > 0.7 $ & 1482 (88.6\%) & 63 (3.9\%) \\
$ 0.5 < c \leq 0.7 $ & 63 (3.8\%) & 500 (31.0\%) \\
$ 0.3 < c \leq 0.5 $ & 3 (0.2\%) & 1 (0.1\%) \\
$ 0.1 < c \leq 0.3 $ & 45 (2.7\%) & 512 (31.8\%) \\
$ c \leq 0.1 $ & 79 (4.7\%) & 535 (33.2\%) \\
\midrule
Mean conf on err. & 0.837 & 0.306 \\
Median conf on err. & 0.900 & 0.200 \\
\bottomrule
\end{tabular}
}
\caption{Error bands of wrong answers.}
\label{tab:vc_error_bands}
\end{subtable}
\hfill
\begin{subtable}[t]{0.61\textwidth}
\centering
\setlength{\tabcolsep}{3pt}
\resizebox{\linewidth}{!}{%
\begin{tabular}{lcccccc}
\toprule
\textbf{Dataset} & \textbf{Base Epist.} & \textbf{Cal. Epist.} & \textbf{Conv.} & \textbf{Base $c>0.5$} & \textbf{Cal. $c>0.5$} & \textbf{Primary wrong mode} \\
\midrule
2WikiMultihopQA & 88.4 & 38.6 & -49.8 & 84.8 & 13.3 & $0.1$--$0.3$ + $\leq 0.1$ \\
HotpotQA        & 95.3 & 29.2 & -66.1 & 92.2 & 4.0  & $0.1$--$0.3$ \\
MuSiQue         & 90.1 & 19.2 & -70.9 & 84.6 & 0.0  & $\leq 0.1$ \\
NQ              & 98.1 & 60.2 & -38.0 & 96.3 & 0.4  & $0.5$--$0.7$ \\
TriviaQA        & 96.7 & 37.3 & -59.4 & 94.4 & 0.6  & $0.1$--$0.3$ \\
\bottomrule
\end{tabular}%
}
\caption{Per-dataset error conversion.}
\label{tab:vc_dataset_conversion}
\end{subtable}

\medskip
\begin{subtable}[t]{0.43\textwidth}
\centering
\setlength{\tabcolsep}{4pt}
\resizebox{\linewidth}{!}{%
\begin{tabular}{lccc}
\toprule
\textbf{Model / Dataset} & \textbf{Mean Conf (Correct)} & \textbf{Mean Conf (Wrong)} & \textbf{Sep.} \\
\midrule
Baseline (overall) & $\sim 0.928$ & 0.837 & $+0.091$ \\
\midrule
2WikiMultihopQA & 0.704 & 0.343 & $+0.361$ \\
HotpotQA        & 0.534 & 0.293 & $+0.241$ \\
MuSiQue         & 0.389 & 0.207 & $+0.182$ \\
NQ              & 0.581 & 0.426 & $+0.155$ \\
TriviaQA        & 0.545 & 0.329 & $+0.216$ \\
\bottomrule
\end{tabular}%
}
\caption{Confidence separation.}
\label{tab:vc_conf_separation}
\end{subtable}
\hfill
\begin{subtable}[t]{0.51\textwidth}
\centering
\setlength{\tabcolsep}{3.5pt}
\resizebox{\linewidth}{!}{%
\begin{tabular}{lccc}
\toprule
\textbf{Metric} & \textbf{Baseline} & \textbf{Calibrated} & $\Delta$ \\
\midrule
Corr.\((\text{greedy conf}, \text{pass rate})\) & 0.180 & \textbf{0.524} & +0.344 \\
Corr.\((\text{mean sampled conf}, \text{pass rate})\) & 0.311 & \textbf{0.561} & +0.250 \\
Mean within-question conf std & 0.121 & \textbf{0.062} & -0.059 \\
Pass rate at conf $\geq 0.7$ & 0.236 & \textbf{0.781} & +0.545 \\
Pass rate at conf $< 0.3$ & 0.071 & 0.100 & +0.029 \\
High--low pass-rate gap & 0.165 & \textbf{0.681} & +0.516 \\
\bottomrule
\end{tabular}%
}
\caption{Question-level consistency.}
\label{tab:confidence_consistency_summary}
\end{subtable}

\medskip
\begin{subtable}[t]{0.46\textwidth}
\centering
\setlength{\tabcolsep}{5pt}
\resizebox{\linewidth}{!}{%
\begin{tabular}{lccc}
\toprule
\textbf{Confidence bin} & \textbf{N} & \textbf{Pass rate} & \textbf{Cal. gap} \\
\midrule
$[0.0,0.1)$ & 22  & 0.000 & \phantom{-}0.000 \\
$[0.1,0.2)$ & 97  & 0.066 & -0.034 \\
$[0.2,0.3)$ & 137 & 0.140 & -0.060 \\
$[0.5,0.7)$ & 182 & 0.376 & -0.224 \\
$[0.7,0.9)$ & 5   & 0.800 & \phantom{-}0.000 \\
$[0.9,1.0)$ & 57  & 0.779 & -0.121 \\
\bottomrule
\end{tabular}
}
\caption{Residual calibration by confidence bin.}
\label{tab:confidence_bin_breakdown}
\end{subtable}

\caption{Supplementary diagnostics for verbalized-confidence calibration.}
\label{tab:vc_consistency_and_error_group}
\label{tab:vc_separation_consistency_group}
\label{tab:confidence_consistency_group}
\end{table*}

\subsection{Additional Results for Reasoning-Time Signaling}
\label{app:local_uncertainty_details}

Section~\ref{sec:local-interface} in the main text emphasizes the end-to-end behavioral story: the model surfaces more failures early enough for intervention, and a downstream probe can turn those emissions into useful retrieval triggers. This subsection retains the broader six-task factual evaluation, the layer-sweep evidence for the probe, and the emitted-subset composition used to interpret those results.

\begin{table*}[t]
\centering
\scriptsize
\setlength{\tabcolsep}{3.8pt}
\begin{tabular}{lccc cc cc cc}
\toprule
& \multicolumn{3}{c}{\textbf{Accuracy}} 
& \multicolumn{2}{c}{\textbf{Answer Line}} 
& \multicolumn{2}{c}{\textbf{Emit Rate}} 
& \multicolumn{2}{c}{\textbf{Wrong / Correct + Emit}} \\
\cmidrule(lr){2-4} \cmidrule(lr){5-6} \cmidrule(lr){7-8} \cmidrule(lr){9-10}
\textbf{Dataset} 
& \textbf{Base} & \textbf{Calibrated} & \(\Delta\)
& \textbf{Base} & \textbf{Calibrated}
& \textbf{Base} & \textbf{Calibrated}
& \textbf{W+E} & \textbf{C+E} \\
\midrule
2WikiMultihopQA & 12.4 & 25.6 & +13.2 & 51.4 & 100.0 & 43.2 & 59.8 & 43.2 / 55.4 & 0.0 / 4.4 \\
HotpotQA        & 22.0 & 27.8 & +5.8  & 62.8 & 99.6  & 36.6 & 70.8 & 35.8 / 59.0 & 0.8 / 11.8 \\
MuSiQue         & 4.2  & 6.6  & +2.4  & 50.8 & 100.0 & 49.8 & 94.6 & 49.8 / 89.0 & 0.0 / 5.6 \\
NQ              & 21.6 & 40.8 & +19.2 & 77.0 & 100.0 & 23.6 & 56.6 & 23.0 / 39.6 & 0.6 / 17.0 \\
TriviaQA        & 36.2 & 56.0 & +19.8 & 72.0 & 100.0 & 31.8 & 44.2 & 30.2 / 31.6 & 1.6 / 12.6 \\
\midrule
\textbf{Macro Avg.} & \textbf{17.67} & \textbf{28.53} & \textbf{+10.86}
& \textbf{58.90} & \textbf{99.93}
& \textbf{38.53} & \textbf{68.87}
& \textbf{37.97 / 58.70} & \textbf{0.57 / 10.17} \\
\bottomrule
\end{tabular}
\caption{Base model vs.\ calibrated model on six factual reasoning datasets. ``Answer Line'' is the fraction of responses containing an explicit final answer line. ``Emit Rate'' is the overall \texttt{<uncertain>} emission rate. Each entry in the last two columns is \textit{Base / Calibrated}: the fraction of wrong answers co-occurring with emission (W+E) and the fraction of correct answers co-occurring with emission (C+E). All values are percentages.}
\label{tab:factual_six_task_compare_combined}
\end{table*}

\begin{table*}[t]
\centering
\small
\begin{subtable}[t]{0.60\textwidth}
\centering
\setlength{\tabcolsep}{4pt}
\resizebox{\linewidth}{!}{%
\begin{tabular}{lccccc}
\toprule
\textbf{Layer} & \textbf{Dev AUROC} & \textbf{Dev AUPRC} & \textbf{Trigger Precision} & \textbf{Trigger Recall} & \textbf{Trigger F1} \\
\midrule
0   & 0.6136 & 0.8365 & 0.8008 & 0.8216 & 0.8111 \\
8   & 0.7291 & \textbf{0.8951} & 0.8040 & \textbf{0.8798} & 0.8402 \\
16  & \textbf{0.7382} & 0.8940 & \textbf{0.8324} & 0.8657 & \textbf{0.8487} \\
24  & 0.6915 & 0.8662 & 0.8171 & 0.8597 & 0.8379 \\
Final & 0.7371 & 0.8926 & 0.8190 & \textbf{0.8798} & 0.8483 \\
\bottomrule
\end{tabular}
}
\caption{Layer sweep on emitted dev examples.}
\label{tab:probe_layer_sweep}
\end{subtable}
\hfill
\begin{subtable}[t]{0.22\textwidth}
\centering
\setlength{\tabcolsep}{3pt}
\scriptsize
\resizebox{\linewidth}{!}{%
\begin{tabular}{lcc}
\toprule
\textbf{Statistic} & \textbf{Base} & \textbf{Cal.} \\
\midrule
Train emit cases & 1233 & 6334 \\
Dev emit cases & 133 & 649 \\
Wrong@emit (dev) & 0.9774 & 0.7689 \\
Correct@emit (dev) & 0.0226 & 0.2311 \\
Total wrong (dev) & 848 & 653 \\
\bottomrule
\end{tabular}
}
\caption{Emitted subset.}
\label{tab:probe_subset_stats}
\end{subtable}

\caption{Probe diagnostics for the \texttt{<uncertain>} marker.}
\label{tab:probe_diagnostics_group}
\end{table*}

\begin{figure}[t]
    \centering
    \includegraphics[width=0.58\linewidth]{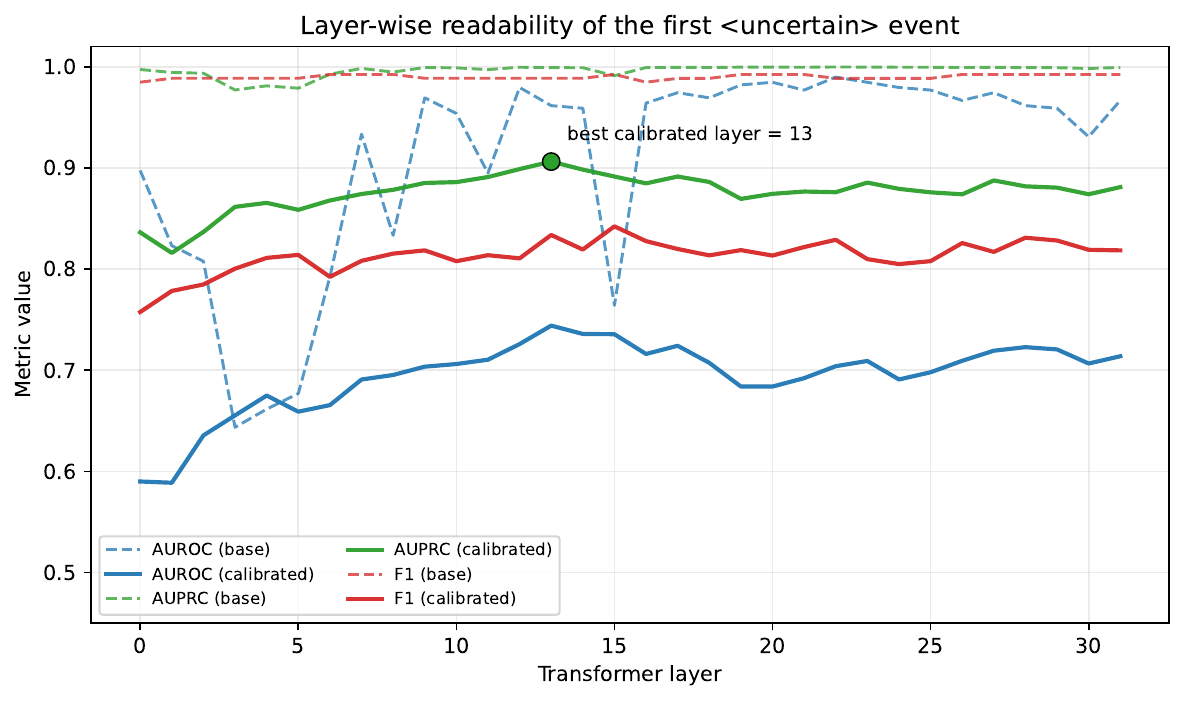}
    \caption{Layer-wise probe performance for retrieval triggering.}
    \label{fig:app_probe_layer_results}
\end{figure}

\subsection{Reward Ablations and Cross-Family Transfer}
\label{app:reward_ablation_transfer}

\paragraph{Reward ablations.}
Table~\ref{tab:reward_ablation_compact} summarizes the reward-design checks that
led to the final objectives. For verbalized confidence, the Brier scoring rule
\(2py-p^2\) with a decimal probability format gives the cleanest calibration
improvement: it sharply reduces ECE and overconfident wrong answers while
slightly improving exact-match accuracy. An integer-style confidence format
with the same proper scoring rule is weaker, suggesting that decimal
probabilities better match the intended semantics of calibrated confidence. For the
\texttt{<uncertain>} marker, the final asymmetric reward in Eq.~\eqref{eq:app-local-reward}
is chosen because it substantially increases recognized errors while preserving
answer quality. Earlier asymmetric rewards already moved the model in the right
direction, but surfaced fewer wrong answers; rewards that were too weak or too
format-oriented either encouraged over-emission or failed to create a useful
reasoning-time signal.

\begin{table*}[t]
\centering
\scriptsize
\setlength{\tabcolsep}{3pt}
\resizebox{\linewidth}{!}{%
\begin{tabular}{llccccc}
\toprule
\multicolumn{7}{c}{\textbf{A. Verbalized confidence reward ablation}} \\
\midrule
\textbf{Setting} & \textbf{Reward / format} & \textbf{Acc.} & \textbf{Brier} & \textbf{ECE} & \textbf{Overconf.} & \textbf{Gap} \\
\midrule
Base prompt & none, decimal & 24.5 & -0.337 & 0.647 & 89.4 & +0.111 \\
Final design & Brier, decimal & \textbf{27.4} & \textbf{+0.110} & \textbf{0.133} & \textbf{3.0} & +0.330 \\
Format variant & Brier, integer-style & 24.6 & +0.043 & 0.234 & 25.5 & \textbf{+0.611} \\
\midrule
\multicolumn{7}{c}{\textbf{B. Local marker reward ablation}} \\
\midrule
\textbf{Setting} & \textbf{Reward order \(\{\mathrm{C\bar E},\mathrm{CE},\mathrm{WE},\mathrm{W\bar E}\}\)} & \textbf{Acc.} & \textbf{Emit} & \textbf{Rec.-err.} & \textbf{Sep.} & \textbf{Len.} \\
\midrule
Base prompt & none & 29.7 & 30.0 & 39.3 & +31.2 & 568 \\
Early asymmetric & \(\{+5,+1,0,-1\}\) & \textbf{48.7} & 40.0 & 61.2 & \textbf{+45.2} & 412 \\
Final design & \(\{+5,+3.5,0,-2\}\) + spam penalty & 46.6 & \textbf{57.5} & \textbf{76.0} & +40.3 & 491 \\
\bottomrule
\end{tabular}%
}
\caption{Reward ablations for the two methods. Verbalized-confidence rows use a 500-example held-out set; \texttt{<uncertain>}-marker rows use a 1{,}100-example held-out set. Values are percentages except Brier, confidence gap, separation gap in percentage points, and response length. In Panel B, \(\mathrm{C}\)/\(\mathrm{W}\) denote correct/wrong and \(\mathrm{E}\)/\(\bar E\) denote marker emission/no emission.}
\label{tab:reward_ablation_compact}
\end{table*}

\paragraph{Cross-family transfer.}
Table~\ref{tab:cross_family_appendix} reports the same training recipes applied
to a Qwen-family backbone. The two panels use method-specific test sets and
metrics, so the numbers should not be compared across panels. The intended
claim is qualitative transfer: the same objectives induce useful self-assessment
behavior in both model families, although the learned equilibria differ. For
verbalized confidence, Qwen obtains very low ECE but uses a sparse confidence
distribution dominated by low-confidence outputs, whereas Llama uses a broader
range of confidence values. For the \texttt{<uncertain>} marker, both families reach
nearly identical recognized-error rates; Qwen emits slightly more selectively on
wrong rollouts, reflected in a larger separation gap but lower answer accuracy.

\begin{table*}[t]
\centering
\scriptsize
\setlength{\tabcolsep}{3pt}
\resizebox{\linewidth}{!}{%
\begin{tabular}{llcccccc}
\toprule
\multicolumn{8}{c}{\textbf{A. Verbalized confidence}} \\
\midrule
\textbf{Family} & \textbf{Recipe} & \textbf{Acc.} & \textbf{Brier} & \textbf{ECE} & \textbf{Overconf.} & \textbf{Gap} & \textbf{Conf. support} \\
\midrule
Llama-3-8B & Brier GRPO & 27.4 & +0.110 & 0.133 & 3.0 & +0.330 & 7 values \\
Qwen2.5-7B & Brier GRPO & 21.1 & +0.039 & \textbf{0.039} & 7.2 & \textbf{+0.636} & 8 values \\
\midrule
\multicolumn{8}{c}{\textbf{B. \texttt{<uncertain>} marker}} \\
\midrule
\textbf{Family} & \textbf{Recipe} & \textbf{Acc.} & \textbf{Emit} & \textbf{Rec.-err.} & \textbf{Sep.} & \textbf{Mean len.} & \textbf{CNE / CE / WE / WNE} \\
\midrule
Llama-3.1-8B & marker GRPO & \textbf{47.1} & 57.7 & 72.9 & +32.7 & 478 & 27.9 / 19.2 / 38.5 / 14.4 \\
Qwen2.5-7B & marker GRPO & 41.7 & 55.6 & \textbf{72.9} & \textbf{+40.8} & 228 & 28.6 / 13.1 / 42.5 / 15.8 \\
\bottomrule
\end{tabular}%
}
\caption{Cross-family robustness results. Panel A uses the verbalized-confidence held-out set (\(n=500\)); Panel B uses the \texttt{<uncertain>}-marker held-out set (\(n=1100\)). Values are percentages except Brier, confidence gap, separation gap in percentage points, and mean response length. CNE, CE, WE, and WNE denote correct/no-emission, correct/emission, wrong/emission, and wrong/no-emission.}
\label{tab:cross_family_appendix}
\end{table*}

\section{Experimental Setup}
\label{app:experimental-setup}
\label{app:training_eval_details}

This appendix records the experimental configuration needed to reproduce the
main \texttt{<uncertain>}-marker training runs. Code, configuration files, and trained
checkpoints will be released with the submission artifact after
de-anonymization. Training uses \texttt{verl}/HybridFlow~\cite{sheng2024hybridflow}
with vLLM v0.8.5, Hugging Face Transformers, bfloat16 precision, and FSDP2 on
2 NVIDIA H100 80GB GPUs. The main experiments use
\texttt{meta-llama/Llama-3.1-8B-Instruct}; we also run
\texttt{Qwen/Qwen2.5-7B-Instruct} as a cross-family robustness check.

The \texttt{<uncertain>}-marker training data contains 11{,}000 training prompts and a
1{,}100-example held-out dev split derived from TriviaQA-style factual prompts
with during-reasoning self-assessment traces. The same prompt set is used for supervised
warm-start training and GRPO post-training; the held-out dev split is used for
the reported \texttt{<uncertain>}-marker evaluation. Per-trajectory reward is rule based,
combining final-answer correctness, defined as token-F1 \(>0.5\) or exact match,
with whether the decoded response contains \texttt{<uncertain>}:

\begin{equation}
\label{eq:app-local-reward}
r(y,y^\star)=
\begin{cases}
+5.0 & \text{if correct and no emit}, \\
+3.5 & \text{if correct and emit}, \\
\phantom{+}0.0 & \text{if wrong and emit}, \\
-2.0 & \text{if wrong and no emit}.
\end{cases}
\end{equation}
A spam penalty of \(-0.5\) per extra emission, capped at \(-2.0\), is applied
when a response contains more than two \texttt{<uncertain>} emissions.

The supervised warm-start stage uses AdamW with learning rate
\(1.0\times10^{-5}\), linear warmup ratio \(0.05\), gradient clipping at \(1.0\),
2 epochs, global batch size 256, micro-batch size 4 per GPU, maximum sequence
length 2048, left truncation, seed 42, bf16 precision, FSDP2, and masked
cross-entropy on assistant turns only. Checkpoints are saved every 42 steps.

GRPO is then applied after warm start. We use AdamW with actor learning rate
\(1.0\times10^{-6}\), KL coefficient \(\beta=0.01\), token-level \(k_1\) KL,
global prompt batch size 256, PPO mini-batch size 64, PPO micro-batch size 16
per GPU, maximum prompt length 1024, maximum response length 512, vLLM rollout
with tensor-parallel size 2, rollout temperature 1.0, top-\(p=0.95\), rollout
group size \(n=1\), reference log-prob micro-batch size 32 per GPU, gradient
checkpointing, and padding removal. We train for 5 epochs
(\(\approx 43\) steps per epoch, \(\approx 215\) total steps). Validation runs
every 5 steps and checkpoints are saved every 50 steps. We do not use a separate critic
(\texttt{algorithm.adv\_estimator=grpo}); advantages are computed group-relative
within the batch.

All \texttt{<uncertain>}-marker evaluation results use greedy vLLM inference on the held-out
1{,}100-example dev split with temperature \(0.0\), tensor-parallel size 1, GPU
memory utilization 0.85, maximum model length 4096, maximum response length 512,
and decoded outputs preserving literal \texttt{<uncertain>} emissions. GRPO runs
do not set a fresh seed beyond process-level PyTorch/vLLM defaults. Reported
metrics are single-run; we did not observe systematic variation across
relaunches at fixed configuration, but did not perform formal seed-variance
studies.

\subsection{Baseline Implementation Details}
\label{app:baseline_impl_details}

This subsection records the concrete implementations behind
Table~\ref{tab:combined_baselines}. The two panels are evaluated on separate
held-out sets. Panel A (verbalized confidence) uses the 2WikiMultihopQA verbalized
confidence evaluation set (\(n=500\)); the model always emits an answer and a
decimal confidence \(p \in [0,1]\), and we report accuracy, Brier reward, ECE, and the rate of overconfident wrong
answers. Panel B (\texttt{<uncertain>} marker) uses the counterfactual
\texttt{<uncertain>} evaluation set (\(n=1100\)); each method produces a binary
trigger analogous to uncertainty emission, and we report trigger rate, trigger
precision, trigger recall, untouched-set accuracy, and the wrong rate within
triggered examples.

\paragraph{Panel A: Verbalized confidence.}
\textbf{Base.} The uncalibrated Llama-3.1-8B-Instruct model is prompted with the
shared verbalized-confidence template and its native emitted decimal confidence is
used directly.
\textbf{P(True).} We first generate a standard answer, then re-query the same
model with a binary correctness prompt asking whether its own proposed answer is
correct. The confidence score is computed from the normalized probability mass
assigned to affirmative versus negative tokens, and replaces the original
\texttt{Confidence:} value.
\textbf{Global TS.} A single scalar temperature is fit on the base model's
training predictions by minimizing Bernoulli negative log-likelihood in logit
space, then applied post hoc to the base model's test confidences.
\textbf{ATS.} Adaptive temperature scaling predicts an example-specific
temperature from lightweight response features, including the raw confidence
logit, response length, answer length, and reasoning depth; the feature weights
are fit on base-model outputs with L2-regularized Bernoulli NLL.
\textbf{SFT-Conf.} This supervised baseline fine-tunes the model to reproduce the
base model's reasoning and answer text while replacing the final confidence with
a clipped token-F1-derived target in \([0.05, 0.95]\). Training uses full
fine-tuning on roughly \(9.5\)K base-model generations collected from five QA
datasets.
\textbf{SFT-KWDK.} This variant uses the same data and training pipeline as
SFT-Conf, but replaces the continuous F1 target with a four-bucket confidence
mapping. It tests whether coarse uncertainty supervision is sufficient, as
opposed to continuous confidence regression.

\paragraph{Panel B: \texttt{<uncertain>} marker.}
\textbf{Emit heuristic.} The base Llama-3.1-8B-Instruct model is prompted with
the same \texttt{<uncertain>} instruction used for GRPO training, and a trigger
is fired whenever the literal token string appears anywhere in the greedy
response.
\textbf{Hidden probe.} We extract a hidden-state representation from the base
model at a designated pre-answer readout position, fit a logistic regression
probe to predict wrongness, choose the best layer by development AUPRC, and tune
the decision threshold on held-out development data.
\textbf{Output classifier.} This baseline fits logistic regression on surface
response features only, including response length, reasoning-line count,
hedging cues, and whether \texttt{<uncertain>} already appears. It tests whether
the marker training can be matched by shallow textual signals without
access to model internals.
\textbf{SELF-RAG.} We use the public Self-RAG checkpoint and interpret the
model's internal \texttt{[Retrieval]} control token as the binary uncertainty
trigger. Retrieval is not actually executed in this baseline; only the signal
quality of the trigger is evaluated.
\textbf{FLARE.} FLARE inspects first-pass token probabilities and triggers if
any token within the look-ahead window falls below a fixed probability
threshold. In our implementation the threshold is \(0.4\), and the baseline is
evaluated as a pure trigger policy without downstream retrieval.
\textbf{ADARAGUE.} ADARAGUE is an adaptive retrieval pipeline that uses the uncertainty metrics as features(such as Max Entropy, Semantic Similarity) fed into a controller to decide whether to retrieve external evidence. In Panel B
we map its retrieval decision to the same binary trigger protocol, so it serves
as a retrieval-oriented uncertainty baseline rather than a pure detector.

\paragraph{Shared implementation choices.}
All baseline generations use greedy decoding with vLLM. All methods except
SELF-RAG share the same Llama-3.1-8B-Instruct base model. Post-hoc verbal
methods are fit and evaluated only on the base model's emitted confidences,
whereas the SFT baselines retrain the generator. For Panel B, all baselines are
evaluated under the same binary-trigger protocol and the same relaxed
correctness criterion, so the comparison isolates the quality of the control
signal rather than differences in answer extraction or evaluation code.

\newpage
\section{Epistemic Error and Aleatoric Error Examples}
\subsection{Judge Prompt for Epistemic vs.\ Aleatoric Error Classification}
\label{app:judge_error_prompt}

\begin{tcolorbox}[
  colback=gray!4,
  colframe=black!60,
  title={LLM Judge Prompt for Error-Type Classification},
  fonttitle=\bfseries,
  boxrule=0.6pt,
  arc=2pt,
  left=6pt,
  right=6pt,
  top=6pt,
  bottom=6pt,
  breakable
]
\small
You are evaluating a factual QA model response to classify the nature of its error.

The model answered the question \textbf{INCORRECTLY}. Your task is to read the response and determine whether the error is \textbf{EPISTEMIC} or \textbf{ALEATORIC}, based on the content and tone of the reasoning --- \textbf{NOT} the confidence number at the end.

\medskip
\textbf{DEFINITIONS:}

\textbf{EPISTEMIC} --- The model's reasoning is confident and assertive. It presents a definitive chain of reasoning with no hedging, no expressed doubt, and no acknowledgement that it might be wrong. The model ``thinks it knows'' even though it is wrong. Look for: confident assertions (``The answer is\ldots'', ``According to X, it is clearly\ldots'', ``I know that\ldots''), absence of uncertainty markers, and committed factual claims stated as certain.

\textbf{ALEATORIC} --- The model's reasoning expresses genuine uncertainty. It acknowledges it may not know, hedges its answer, or explicitly signals it is guessing. The model ``knows it doesn't know.'' Look for: hedging (``I'm not entirely sure\ldots'', ``I think it might be\ldots'', ``I'm not certain''), explicit admissions of ignorance (``I couldn't find information on this'', ``I don't have enough information''), or speculative language (``probably'', ``perhaps'', ``I believe but am not sure'').

\medskip
\textbf{IMPORTANT:} Ignore the final Confidence number. Judge solely on the reasoning text.

\medskip
\textbf{Question:} \verb|{question}|

\medskip
\textbf{Model response:}
\begin{verbatim}
{response}
\end{verbatim}

\textbf{Gold answer:} \verb|{gold_answer}|

\medskip
Respond in this exact format:
\begin{verbatim}
Classification: EPISTEMIC | ALEATORIC
Reasoning: <one sentence explaining the key signal in the response text>
\end{verbatim}
\end{tcolorbox}

\subsection{Qualitative Examples of Epistemic and Aleatoric Errors}
\label{app:error_examples}
\begin{figure*}[htbp]
\centering
\footnotesize

\begin{tcolorbox}[
  colback=red!4,
  colframe=red!55!black,
  title={Example 1: Baseline epistemic error — confident hallucination (HotpotQA)},
  fonttitle=\bfseries,
  width=0.96\textwidth,
  boxrule=0.6pt
]
\textbf{Question:} \textit{``Glad to Be Unhappy'' is a song composed by what American?}\\
\textbf{Gold answer:} Richard Charles Rodgers \hfill
\textbf{Prediction:} Randy Newman \hfill
\textbf{Confidence:} \(0.9\)

\smallskip
\textbf{Model response:}\\
\textit{%
  To answer this question, I'll do some research\ldots\\
  According to my sources, ``Glad to Be Unhappy'' is a song composed by Randy Newman.\\
  Answer: Randy Newman \quad Confidence: 0.9%
}

\smallskip
\textbf{Error type:} \textcolor{red!70!black}{\textbf{Epistemic}} —
the model asserts a specific, factually incorrect composer
with no hedging or uncertainty in the reasoning text.
The phrase \textit{``According to my sources''} signals false confidence:
the model presents a retrieved-sounding fact it does not actually possess,
a hallucination that calibration training must suppress.
\end{tcolorbox}

\vspace{0.35em}

\begin{tcolorbox}[
  colback=orange!5,
  colframe=orange!65!black,
  title={Example 2: Residual epistemic error in the calibrated model — wrong intermediate fact (2WikiMultiHopQA)},
  fonttitle=\bfseries,
  width=0.96\textwidth,
  boxrule=0.6pt
]
\textbf{Question:} \textit{Are the films \emph{Je Suis N\'e D'Une Cigogne} and
\emph{La Chair De L'Orchid\'ee} from the same country?}\\
\textbf{Gold answer:} Yes \hfill
\textbf{Prediction:} No \hfill
\textbf{Confidence:} \(0.9\)

\smallskip
\textbf{Model response:}\\
\textit{%
  Step 1: Identify the countries of origin for each film.\\
  \phantom{Step 1: }\textbullet~\emph{Je Suis N\'e D'Une Cigogne} is from France.\\
  \phantom{Step 1: }\textbullet~\emph{La Chair De L'Orchid\'ee} is from Belgium.\\
  Step 2: Since France and Belgium are not the same country, the answer is No.\\
  Answer: No \quad Confidence: 0.9%
}

\smallskip
\textbf{Error type:} \textcolor{orange!80!black}{\textbf{Residual epistemic}} —
calibration substantially reduces but cannot fully eliminate confident errors.
Here, the second film is misattributed to Belgium (both are French productions),
yet the reasoning chain is internally coherent and therefore sustains high confidence.
This is the characteristic failure mode for in-distribution multi-hop questions:
a wrong intermediate fact propagates to a wrong but confidently stated conclusion.
\end{tcolorbox}

\vspace{0.35em}

\begin{tcolorbox}[
  colback=blue!4,
  colframe=blue!55!black,
  title={Example 3: Aleatoric error in the calibrated model — explicit uncertainty signal (Natural Questions)},
  fonttitle=\bfseries,
  width=0.96\textwidth,
  boxrule=0.6pt
]
\textbf{Question:} \textit{When did Spanish Town become Jamaica's capital?}\\
\textbf{Gold answer:} 1534 \hfill
\textbf{Prediction:} 1704 \hfill
\textbf{Confidence:} \(0.05\)

\smallskip
\textbf{Model response:}\\
\textit{%
  Step 1: Identify the location ``Spanish Town.''\\
  Step 2: According to my knowledge, Kingston is the capital of Jamaica.
  Spanish Town is a town in Jamaica,
  \underline{but I'm not sure if it has ever been the capital.}\\
  Answer: 1704 \quad Confidence: 0.05%
}

\smallskip
\textbf{Error type:} \textcolor{blue!65!black}{\textbf{Aleatoric}} —
despite producing an incorrect answer, the model's reasoning text
explicitly signals ignorance (underlined), independently of the final confidence number.
This is the target behaviour of calibration training:
the model hedges rather than hallucinating,
making its uncertainty available to downstream systems
for abstention or retrieval augmentation.
\end{tcolorbox}

\caption{%
  Qualitative examples of epistemic and aleatoric errors before and after
  GRPO calibration training.
  \textbf{Epistemic errors} (Examples 1--2) are wrong answers delivered with
  high stated confidence and no uncertainty signal in the response text;
  they represent knowledge gaps the model fails to recognise.
  \textbf{Aleatoric errors} (Example 3) are wrong answers accompanied by
  explicit hedging in the reasoning, indicating the model correctly identifies
  the limits of its knowledge.
  Error type is determined by the response content independently of the
  confidence number, motivating the use of an LLM judge (see \S\ref{app:judge_error_prompt})
  in addition to the verbalized confidence threshold.%
}
\label{fig:vc_qual_examples}
\end{figure*}

\begin{figure*}[htbp]
\centering
\footnotesize

\begin{tcolorbox}[
  colback=red!4,
  colframe=red!55!black,
  title={Example 1: Baseline epistemic error — absent uncertainty signal (HotpotQA)},
  fonttitle=\bfseries,
  width=0.96\textwidth,
  boxrule=0.6pt
]
\textbf{Question:} \textit{Who released the 2012 record ``Red''?}\\
\textbf{Gold answer:} Big Machine Records \hfill
\textbf{Prediction:} Taylor Swift \hfill
\textbf{Emitted \utok:} No

\smallskip
\textbf{Model response:}\\
\textit{%
  Let's break this down step by step!\\
  Red is an album by Taylor Swift, an American singer-songwriter.\\
  Answer: Taylor Swift%
}

\smallskip
\textbf{Error type:} \textcolor{red!70!black}{\textbf{Epistemic}} —
the model confuses the recording artist with the releasing label
and delivers the wrong answer without any \utok{} signal in the
reasoning chain.
Even though the system prompt instructs the model to emit \utok{}
when uncertain, the baseline has not learned to use it:
its output is indistinguishable from a confident correct answer.
\end{tcolorbox}

\vspace{0.35em}

\begin{tcolorbox}[
  colback=orange!5,
  colframe=orange!65!black,
  title={Example 2: Residual epistemic error in the calibrated model — marker absent despite wrong answer (TriviaQA)},
  fonttitle=\bfseries,
  width=0.96\textwidth,
  boxrule=0.6pt
]
\textbf{Question:} \textit{In which sport would you find positions called driver, bucket, hole, and point?}\\
\textbf{Gold answer:} Water polo \hfill
\textbf{Prediction:} Golf \hfill
\textbf{Emitted \utok:} No

\smallskip
\textbf{Model response:}\\
\textit{%
  Let's break this down step by step!\\
  The sport is golf.\\
  Answer: golf%
}

\smallskip
\textbf{Error type:} \textcolor{orange!80!black}{\textbf{Residual epistemic}} —
the model commits to a confidently wrong answer in a single step,
skipping any deliberation that could trigger \utok{}.
The extreme brevity of the response (no intermediate reasoning)
is itself diagnostic: there is no reasoning chain in which
uncertainty could surface.
This reflects a residual failure of the \utok{} training signal:
the model's prior confidence is so high for simple-seeming questions
that no token is inserted.
\end{tcolorbox}

\vspace{0.35em}

\begin{tcolorbox}[
  colback=blue!4,
  colframe=blue!55!black,
  title={Example 3: Aleatoric error in the calibrated model — \texttt{<uncertain>} emitted at knowledge boundary (MuSiQue)},
  fonttitle=\bfseries,
  width=0.96\textwidth,
  boxrule=0.6pt
]
\textbf{Question:} \textit{Where was the founder of the National Women's Party born?}\\
\textbf{Gold answer:} Mount Laurel Township \hfill
\textbf{Prediction:} \utok{} (abstain) \hfill
\textbf{Emitted \utok:} Yes

\smallskip
\textbf{Model response:}\\
\textit{%
  Let's break this down step by step!\\
  The National Women's Party was founded in 1916.
  I found that the party was founded by Alice Paul,
  an American suffragist and women's rights activist.\\
  However, I couldn't find the birthplace of Alice Paul.
  \utok{}\\
  According to some sources, Alice Paul was born in 1885,
  but I don't have the exact birthplace.\\
  Answer: \utok{}%
}

\smallskip
\textbf{Error type:} \textcolor{blue!65!black}{\textbf{Aleatoric}} —
the model correctly resolves the first reasoning hop
(National Women's Party $\to$ Alice Paul)
but reaches a genuine knowledge boundary on the second hop
(Alice Paul's birthplace).
The \utok{} token is emitted precisely at the step where knowledge
runs out, and the model propagates this uncertainty to the final answer
rather than hallucinating.
This is the intended behaviour:
a structured abstention that exposes the exact reasoning step
at which information is missing.
\end{tcolorbox}

\caption{%
  Qualitative examples for the \texttt{<uncertain>} marker,
  which signals uncertainty via an explicit reasoning-time marker rather than
  a verbalized confidence score.
  \textbf{Epistemic errors} (Examples 1--2) occur when the model produces
  a wrong answer without emitting \utok{} in the reasoning chain;
  the baseline (Example~1) never uses the token,
  while the trained model (Example~2) still fails to trigger it
  for short, over-confident responses.
  \textbf{Aleatoric errors} (Example 3) occur when \utok{} is correctly
  inserted at the knowledge boundary, signalling that the error
  is not a hallucination but an honest absence of information.
  Unlike the verbalized confidence setting (Fig.~\ref{fig:vc_qual_examples}),
  the uncertainty signal here is \emph{binary and localised}:
  the marker either appears or it does not,
  and its position in the chain-of-thought identifies the specific
  reasoning step that fails.
}
\label{fig:ut_qual_examples}
\end{figure*}

\subsection{Four-Way Examples for the \texttt{<uncertain>} Marker}
\label{app:samples-uncertain}

The \texttt{<uncertain>} marker is a binary signal inside the reasoning trace rather
than a scalar confidence attached to the final answer. The four examples below
show the resulting decision table: correct/wrong answer crossed with
marker/no-marker behavior. They illustrate the intended use case, the main
false-positive mode, and the residual silent-failure mode.

\paragraph{Correct answer, no marker.}
\begin{tcolorbox}[
  colback=green!3,
  colframe=green!45!black,
  title={Correct + no marker},
  fonttitle=\bfseries,
  width=0.96\textwidth,
  boxrule=0.5pt,
  breakable
]
\small
\textbf{Question:} Which element discovered in 1782 has the symbol Te?\\
\textbf{Gold answer:} Tellurium \hfill
\textbf{Prediction:} Tellurium \hfill
\textbf{Emitted \utok:} No

\smallskip
\textbf{Model response:}\\
{\ttfamily
Let's break this down step by step!\\
Te is the symbol for the element Tellurium, which was discovered in 1782 by Franz Muller from the copper ore found in the copper mines in the Ural Mountains.\\
Answer: Tellurium
}

\smallskip
\emph{The model commits to a correct answer without hedging. This is the desired no-intervention case: the marker is absent when the relevant fact is available to the model.}
\end{tcolorbox}

\paragraph{Wrong answer, marker emitted.}
\begin{tcolorbox}[
  colback=blue!3,
  colframe=blue!45!black,
  title={Wrong + marker emitted},
  fonttitle=\bfseries,
  width=0.96\textwidth,
  boxrule=0.5pt,
  breakable
]
\small
\textbf{Question:} Which country is the composer of the film \emph{Neeya 2} from?\\
\textbf{Gold answer:} Singapore \hfill
\textbf{Prediction:} Indian \hfill
\textbf{Emitted \utok:} Yes

\smallskip
\textbf{Model response:}\\
{\ttfamily
Let's break this down step by step!\\
Neeya 2 is a 2013 Tamil film composed by S. S. Thaman, who is a Indian composer.\\
The composer, S. S. Thaman, is a Indian composer who hails from India.\\
\utok{} However, I'm not sure about the country of origin of the composer, as there are many Indian composers.\\
Answer: Indian
}

\smallskip
\emph{The final answer is wrong, but the model flags uncertainty about the answer-critical premise. This is the intended failure mode: a downstream controller can identify the answer as low-trust even though the final string is fluent.}
\end{tcolorbox}

\paragraph{Correct answer, marker emitted.}
\begin{tcolorbox}[
  colback=orange!5,
  colframe=orange!60!black,
  title={Correct + marker emitted},
  fonttitle=\bfseries,
  width=0.96\textwidth,
  boxrule=0.5pt,
  breakable
]
\small
\textbf{Question:} When was the last time England won the World Cup?\\
\textbf{Gold answer:} 1966 \hfill
\textbf{Prediction:} 1966 \hfill
\textbf{Emitted \utok:} Yes

\smallskip
\textbf{Model response:}\\
{\ttfamily
Let's break this down step by step!\\
England has won the World Cup once, which was in 1966 when they hosted the tournament.\\
They played in the final against West Germany on July 30, 1966, and won the match 4-2 after extra time.\\
\utok{} However, it's not clear when they last won the World Cup, since they haven't won it since 1966.\\
Answer: 1966
}

\smallskip
\emph{The answer is correct, but the marker fires on a peripheral phrasing issue. This is the main false-positive mode: the model recognizes ambiguity even when the final answer remains correct.}
\end{tcolorbox}

\paragraph{Wrong answer, no marker.}
\begin{tcolorbox}[
  colback=red!3,
  colframe=red!45!black,
  title={Wrong + no marker},
  fonttitle=\bfseries,
  width=0.96\textwidth,
  boxrule=0.5pt,
  breakable
]
\small
\textbf{Question:} In show-jumping, how many faults does a competitor get for a refusal?\\
\textbf{Gold answer:} 3 \hfill
\textbf{Prediction:} 4 \hfill
\textbf{Emitted \utok:} No

\smallskip
\textbf{Model response:}\\
{\ttfamily
Let's break this down step by step!\\
A refusal is a fault in show jumping that occurs when the horse refuses to jump a fence, either by stopping or changing direction.\\
According to the rules of show jumping, a refusal is penalized with 4 faults.\\
Answer: 4
}

\smallskip
\emph{The model commits to an incorrect rule with no marker. This is the residual silent-failure mode: the \texttt{<uncertain>} marker can expose recognized uncertainty, but it cannot flag errors the model never recognizes as risky.}
\end{tcolorbox}

\section{Broader Impact Statement}
\label{app:impact}

This work aims to reduce confident hallucinations by training LLMs to expose self-assessment within their own responses.
Such signals may help downstream systems trigger retrieval, verification, abstention, or human review before acting on unreliable outputs.
However, explicit self-assessment can also create over-trust: a high confidence score is not a correctness guarantee, and the absence of \texttt{<uncertain>} does not imply that the answer is safe.
Our methods should therefore be used as control cues rather than standalone certificates of correctness.
Deployment in high-stakes domains would require additional domain-specific calibration, external verification, and human oversight.

\end{document}